\documentclass{article}

     \usepackage[preprint,nonatbib]{neurips_2020}

\usepackage[utf8]{inputenc} \usepackage[T1]{fontenc}    \usepackage{hyperref}       \usepackage{url}            \usepackage{booktabs}       \usepackage{amsfonts}       \usepackage{microtype}      \usepackage{color, colortbl}
\usepackage{algorithm, algorithmic}
\usepackage{amsmath, amssymb, amsthm}
\usepackage{enumitem}
\usepackage{mathtools}
\usepackage{nicefrac}
\usepackage{subfig}
\usepackage{mwe}

\definecolor{Gray}{gray}{0.9}

\newtheorem{lemma}{Lemma}
\newtheorem{theorem}{Theorem}
\newtheorem{corollary}{Corollary}
\newtheorem{appendix_theorem}{Theorem}

\DeclareMathOperator*{\argmax}{arg\,max}
\DeclareMathOperator*{\argmin}{arg\,min}

\newcommand{\ba}{\mathbf{a}}
\newcommand{\bb}{\mathbf{b}}
\newcommand{\bg}{\mathbf{g}}
\newcommand{\bl}{\mathbf{l}}
\newcommand{\bp}{\mathbf{p}}
\newcommand{\bu}{\mathbf{u}}
\newcommand{\bv}{\mathbf{v}}
\newcommand{\bw}{\mathbf{w}}
\newcommand{\bx}{\mathbf{x}}
\newcommand{\by}{\mathbf{y}}
\newcommand{\bA}{\mathbf{A}}
\newcommand{\bB}{\mathbf{B}}
\newcommand{\bE}{\mathbf{E}}
\newcommand{\bI}{\mathbf{I}}

\newcommand{\calA}{\mathcal{A}}
\newcommand{\calB}{\mathcal{B}}
\newcommand{\calD}{\mathcal{D}}

\newcommand{\calY}{\mathcal{Y}}
\newcommand{\calN}{\mathcal{N}}

\title{The Trade-Offs of Private Prediction}

\author{
 Laurens van der Maaten$^{*}$ \quad\quad\quad\quad Awni Hannun\thanks{Both authors contributed equally to the paper.}\\
  Facebook AI Research, New York\\
  \texttt{\{lvdmaaten,awni\}@fb.com} \\
}

\begin{document}

\maketitle

\begin{abstract}
Machine learning models leak information about their training data every time they reveal a prediction.
This is problematic when the training data needs to remain private.
Private prediction methods limit how much information about the training data is leaked by each prediction.
Private prediction can also be achieved using models that are trained by private training methods.
In private prediction, both private training and private prediction methods exhibit trade-offs between privacy, privacy failure probability, amount of training data, and inference budget.
Although these trade-offs are theoretically well-understood, they have hardly been studied empirically.
This paper presents the first empirical study into the trade-offs of private prediction.
Our study sheds light on which methods are best suited for which learning setting.
Perhaps surprisingly, we find private training methods outperform private prediction methods in a wide range of private prediction settings.
\end{abstract}

\section{Introduction}
\label{sec:introduction}
Machine learning models are frequently trained on data that needs to remain private, even though the predictions produced by those models are revealed to the outside world.
For example, a hotel recommendation model may be trained on hotel reservation data that its users should not have access to.
Unless proper care is taken, a user of a hotel recommendation service may be able to extract hotel reservation data of other users from the recommendations they receive, which would result in a privacy violation~\cite{carlini2019secret, fredrikson2015model,sablayrolles2019,shokri2017membership,yeom2018}.
Such privacy violations may happen even when users do not have direct access to the model parameters but only to model predictions: several studies have demonstrated that it is possible to reconstruct model parameters from a series of model predictions~\cite{carlini2020,milli2019,tramer2016}.

The goal of \emph{private prediction} is to prevent such privacy violations by limiting the amount of information about the training data that can be obtained from a series of model predictions~\cite{dwork2018}.
Private prediction methods~\cite{bassily2018,dagan2019,dwork2018,nandi2019} perturb the predictions of non-private models to obfuscate any information that can be inferred about the data used to trained those models.
\emph{Private training} methods~\cite{abadi2016,bassily2014,chaudhuri2011,kifer2012,mironov2019,wu2017,wang2017} can be used to obtain private predictions as well.
In particular, private training guarantees that model parameters reveal little information about the training data.
As a result, the predictions of privately trained models do not leak information about the training data either. 

Both private prediction and private training methods exhibit a range of trade-offs when they are used in a private prediction setting.
Specifically, there exist trade-offs between the accuracy of the predictions, the privacy level that can be guaranteed, the probability of a privacy failure, the number of predictions that is revealed publicly (\emph{i.e.}, the inference budget), and the amount of training data.
While these trade-offs are theoretically well understood, little is known about them empirically. 
Prior empirical studies evaluate only private training methods and do not consider the private prediction setting~\cite{iyengar2019, jayaraman2019}.
This paper performs an empirical study into the trade-offs of private prediction.
We aim to provide guidance to practitioners on which private prediction methods are most suitable for a given learning setting.
Perhaps surprisingly, we find that private training methods offer a better privacy-accuracy trade-off than private prediction methods in many practical learning settings.

\section{Problem Statement}
\label{sec:problem_statement}
Consider a private machine learning model $\phi(\bx; \theta)$ with parameters $\theta$ that given a $D$-dimensional input vector, $\bx \in \mathbb{R}^D$, produces a probability vector over $C$ classes $\by \in \Delta^C$.
Herein, $\Delta^C$ represents the $(C\!-\!1)$-dimensional probability simplex.
The parameters $\theta$ were obtained by fitting the model on a training set of $N$ labeled examples, $\calD = \{(\bx_1, \by_1), \dots, (\bx_N, \by_N) \}$, that needs to remain private.

We consider the common scenario in which machine learning model $\phi(\cdot)$ is provided as a service to other parties by the model owner.
Specifically, someone provides a vector $\hat{\bx}$ to the owner of the model, who uses it to compute $\hat{\by} = \phi(\hat{\bx}; \theta)$ and publicly reveals the prediction $\hat{\by}$.
From the perspective of the model owner, there is an inherent risk here that whoever observes $(\hat{\bx}, \hat{\by})$ obtains information about the private training set $\calD$ through the model query: prediction $\hat{\by}$ carries information about parameters $\theta$ that, in turn, carries information about the private training set $\calD$.

The model owner is interested in limiting the amount of information that others can learn about the private training set $\calD$ via an \emph{inference budget} of $B$ queries, $\mathcal{Q} =\{\hat{\bx}_1, \dots, \hat{\bx}_B\}$, to model $\phi(\cdot)$.
Specifically, the model owner aims to provide an $(\epsilon,\delta)$-differential privacy guarantee~\cite{dwork2006} on the information that is leaked about $\calD$ by releasing $B$ predictions:
\begin{equation}
  Pr\left[ \phi\left(\hat{\bx}_b; \theta(\calD)\right)  \subseteq \calY_b : b = 1, \ldots, B \right] \leq e^\epsilon Pr\left[ \phi\left(\hat{\bx}_b; \theta(\calD')\right)  \subseteq \calY_b : b = 1, \ldots, B \right] + \delta,
\label{eq:dp}
\end{equation}
for $\epsilon \geq 0$ and $\delta \geq 0$, $\forall \calY_b \subseteq \Delta^{C}$, $\forall \mathcal{Q}$ with $|\mathcal{Q}| = B$, and for all datasets $\calD$ and $\calD'$ that differ in only one training example.
Herein, we adopt the short-hand notation $\theta(\calD)$ to indicate the parameters $\theta$ that were obtained by training the model on training set $\calD$. The model owner can limit the amount of information that the predictions leak about training set $\calD$ in two primary ways:
\begin{enumerate}[leftmargin=*]
\setlength\itemsep{0em}
\item The model owner can perform \textbf{private training}~\cite{bassily2018,dagan2019,dwork2018,nandi2019} of model $\phi(\cdot)$.
Differentially private training guarantees that the model parameters $\theta$ reveal little information about training set $\calD$.
Because differential privacy is closed under post-processing~\cite{dwork2011}, the model owner can reveal $B = \infty$ predictions (or the model parameters) and still maintain the property in Equation~\ref{eq:dp}.

\item The model owner can perform \textbf{private prediction}~\cite{abadi2016,bassily2014,chaudhuri2011,kifer2012,mironov2019,wu2017,wang2017} using a model $\phi'(\cdot)$ that is not itself differentially private.
Private prediction methods construct $\phi(\cdot)$ from $\phi'(\cdot)$ in a way that limits the amount of information that $\phi(\cdot)$ reveals about $\calD$ through $B$ predictions.
\end{enumerate}

\section{Methods}
\label{sec:methods}

This study performs an empirical analysis of both private training and private prediction methods. We adopt a regularized empirical risk minimization framework in which we minimize:
\begin{equation}
    J(\theta; \calD) = \frac{1}{N} \sum_{n=1}^N \ell(\phi'(\bx_n; \theta), \by_n)  + \lambda R(\theta)
\end{equation}
with respect to $\theta$, where $\ell(\cdot)$ is a loss function, $R(\cdot)$ is a regularizer, and $\lambda \geq 0$ is a regularization parameter.
Some of the private prediction methods that we study make one or more of the following assumptions to obtain the privacy guarantee in Equation~\ref{eq:dp}:

\begin{enumerate}[leftmargin=*]
\setlength\itemsep{0em}
    \item The loss function $\ell(\cdot)$ is strictly convex, continuous, and
        differentiable everywhere.\label{as:loss}
    \item The regularizer $R(\theta)$ is $1$-strongly convex, continuous, and
        differentiable everywhere w.r.t. $\theta$. \label{as:regularizer}
    \item The non-randomized model $\phi'(\cdot)$ is linear in $\bx$, that is, $\phi'(\bx; \theta) = \theta^\top \bx$. \label{as:linear}
    \item The loss function $\ell(\cdot)$ is Lipschitz with a constant $K$, that is, $\|
        \nabla l\|_2 \le K$. \label{as:lip}
    \item The inputs are contained in the unit $L_2$ ball, that is, $\|\bx\|_2 \le 1$
          for all $\bx$. \label{as:l2inp}
\end{enumerate}

Detailed algorithms of all methods and proofs of all theorems are presented in the appendix.

\subsection{Private Training}
\label{sec:private_training}

We consider three private training methods in this study: (1) the model sensitivity method, (2) loss perturbation, and (3) differentially private stochastic gradient descent (SGD).

\noindent\textbf{\underline{Model sensitivity.}}
A simple private training method is for the model owner to add noise to the model parameters in order to hide information about the training data that is captured in those parameters.
The model sensitivity method~\cite{chaudhuri2011} constructs differentially private parameters $\theta$ by adding noise to the minimizer of $J(\cdot)$.
Specifically, we extend~\cite{chaudhuri2011} to multi-class linear models with $\theta \in \mathbb{R}^{D\times C}$ and sample the noise matrix $\bB$ from $p(\bB) \propto e^{-\beta \|\bB\|_F}$.
To satisfy the property in Equation~\ref{eq:dp}, we choose $\beta = \frac{N\lambda\epsilon}{2K}$. In practice, we adopt the multi-class logistic loss which has a Lipschitz constant $K=\sqrt{2}$.

\begin{theorem}
  \label{thm:model_perturb}
  Given assumptions \ref{as:loss}, \ref{as:regularizer}, \ref{as:linear},
  \ref{as:lip}, and \ref{as:l2inp}, the model sensitivity method is
  $(\epsilon, 0)$-differentially private.
\end{theorem}

By changing $p(\bB)$ to a zero-mean isotropic Gaussian distribution with standard deviation $\sigma = \frac{2K \alpha}{N\lambda \sqrt{2\epsilon}}$ for an $\alpha$ that depends on $\epsilon$ and $\delta$, we can also obtain $(\epsilon, \delta)$-differentially private models for $\delta > 0$~\cite{balle2018improving}.

\begin{theorem}
  \label{thm:gaussian_model_perturb}
  Given assumptions \ref{as:loss}, \ref{as:regularizer}, \ref{as:linear},
  \ref{as:lip}, and \ref{as:l2inp}, the Gaussian model sensitivity method is
  $(\epsilon, \delta)$-differentially private for $\delta \in (0, 1)$.
\end{theorem}

\noindent\textbf{\underline{Loss perturbation.}}
Rather than perturbing the learned parameters after training, the model owner can instead randomly perturb the loss function that is minimized during training to obtain a differentially private model~\cite{chaudhuri2011,kifer2012}.
We extend the loss perturbation method of~\cite{chaudhuri2011,kifer2012} to multi-class classification by minimizing the following randomly perturbed loss function:
\begin{equation}
    J'(\theta; \calD) = \frac{1}{N} \sum_{n=1}^N \ell(\phi'(\bx_n; \theta), \by_n)  + \frac{\lambda}{N} R(\theta) + \frac{1}{N} \text{tr}(\bB^\top \theta) + \frac{1}{2}\rho \lVert \theta \rVert_F^2,
\end{equation}
where $\text{tr}(\cdot)$ represents the trace of a square matrix.
The noise matrix $\bB$ is sampled from $p(\bB) \propto e^{-\beta \|\bB\|_F}$.
The scale parameter of the noise distribution $\beta = \frac{\epsilon}{2K}$, and the parameter $\rho$ that governs the additional $L_2$ regularization is set such that $\rho \ge \frac{2LC}{\epsilon}$.
The constant $L$ is an upper bound on the eigenvalues of the Hessian of $\ell(\cdot)$, that is, $\lambda_{\max}(\nabla^2 \ell(\theta^\top \bx, \by)) \le L$ for all $\theta, \bx$ and $\by$. We use a multi-class logistic loss, which has a Hessian with eigenvalues bounded by $L = \frac{1}{2}$ (see appendix).

\begin{theorem}
  \label{thm:loss_perturb}
  Given a convex loss function $\ell(\cdot)$, a convex regularizer $R(\cdot)$ with
  continuous Hessians, assumptions \ref{as:linear}, \ref{as:lip}, and
  \ref{as:l2inp}, and assuming that  $\lambda_{\max}(\nabla^2 \ell) \le L$, the
  loss perturbation method is $(\epsilon, 0)$-differentially private.
\end{theorem}

By changing noise distribution $p(\bB)$ into a zero-mean isotropic Gaussian distribution with standard deviation $\sigma = \frac{K}{\epsilon} \sqrt{8 \ln(2 / \delta) + 4\epsilon}$, we can also obtain $(\epsilon, \delta)$-differential privacy for $\delta > 0$~\cite{kifer2012}.

\begin{theorem}
  Given a convex loss $\ell(\cdot)$, a convex regularizer $R(\cdot)$ with
  continuous Hessians, assumptions \ref{as:linear}, \ref{as:lip}, and
  \ref{as:l2inp}, and assuming that  $\lambda_{\max}(\nabla^2 \ell) \le L$, the
  Gaussian loss perturbation method is $(\epsilon, \delta)$-differentially private.
\end{theorem}

\begin{figure*}[!t]\centering
  \subfloat[\label{fig:1a}For privacy failure probability $\delta=0$.]{  \includegraphics[width=0.45\linewidth]{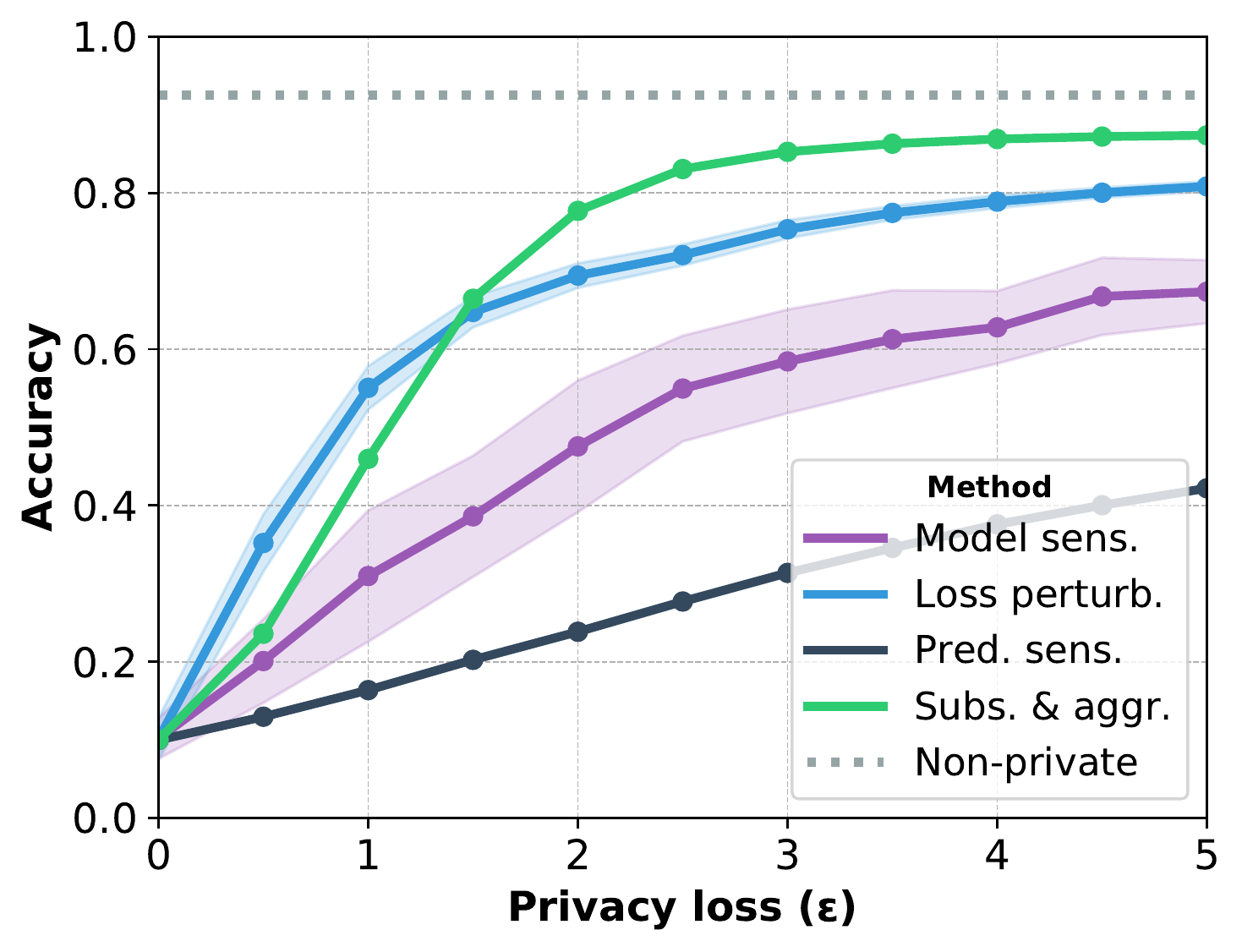}}
\hspace{6mm}
  \subfloat[\label{fig:1b}For privacy failure probability $\delta=10^{-5}$.]{  \includegraphics[width=0.455\linewidth]{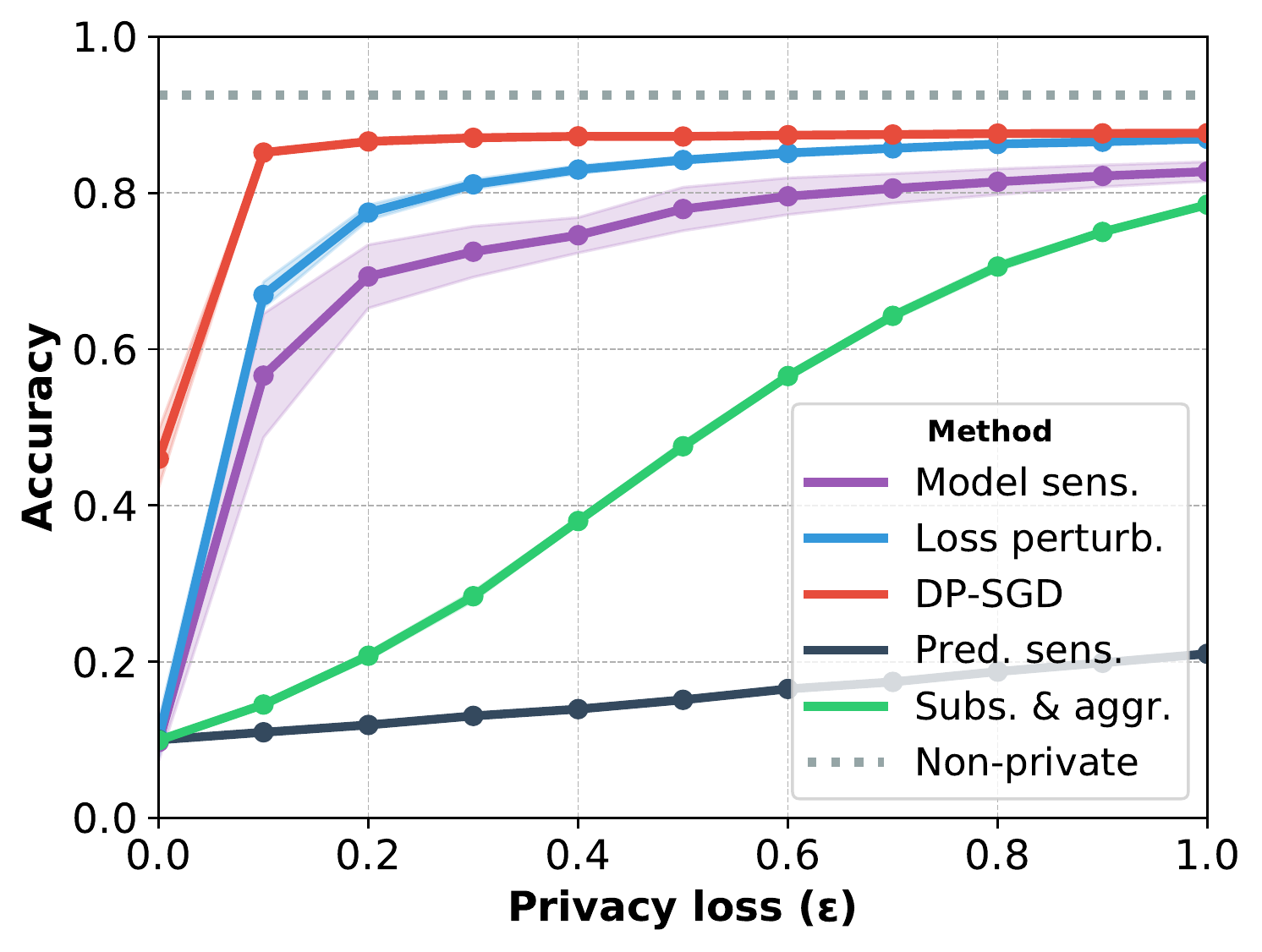}}
\caption{Test accuracy on MNIST dataset as a function of privacy loss $\epsilon$ for inference budget $B=100$. In~\ref{fig:1b}, $\epsilon$ ranges between $0$ and $1$ because of limitations in some methods when $\delta > 0$.}
\label{fig:1}
\end{figure*}

\noindent\textbf{\underline{Differentially private SGD (DP-SGD).}}
Rather than \emph{post hoc} perturbation of the final model parameters or \emph{pre hoc} perturbation of the loss, the DP-SGD method limits the influence of individual examples in parameter updates~\cite{abadi2016}.
DP-SGD draws a batch $\calB$ of training examples uniformly at random and computes the gradient $\bg_n = \frac{\partial J\left(\theta; \{(\bx_n, \by_n)\}\right)}{\partial \theta}$ for each example in the batch.
It clips the resulting per-example gradients to have a norm bound of $\nu$, aggregates the clipped gradients over the batch, and adds Gaussian noise to obtain a private, approximate parameter gradient $\tilde{\bg}$:
\begin{equation}
  \tilde{\bg} = \frac{1}{| \calB |}  \left[ \sum_{n \in \{1, \dots, | \calB |\}} \left(\frac{\bg_n}{\max \left(1, \frac{\lVert \bg_n \rVert_2}{\nu} \right)}\\\right) + \calN(\mathbf{0}, \sigma^2 \nu^2 \bI) \right],
\label{eq:dpsgd}
\end{equation}
where $\sigma^2 \nu^2$ is the variance of the Gaussian, $\mathbf{0}$ is an appropriately sized vector of zeros, and $\bI$ an appropriately sized identity matrix.
The resulting $\tilde{\bg}$ is used to perform the parameter update.
Akin to the model sensitivity method, choosing $\sigma = \frac{1}{\epsilon}\sqrt{2 \ln (1.25 / \delta)}$ makes each $\tilde{\bg}$ vector $(\epsilon, \delta)$-differentially private with respect to the examples in batch $\calB$.
The privacy amplification theorem~\cite{kasiviswanathan2011can} states that $\tilde{\bg}$ is $(q \epsilon, q \delta)$-differentially private with respect to examples from $\calD$, with $q = \nicefrac{ |\calB| }{N}$.

Privacy guarantees on the parameters learnt by DP-SGD after $M$ parameter updates can be obtained via a ``moments accountant''.
We use the moments accountant of~\cite{mironov2019}, which uses an analysis based on R{\'e}nyi differential privacy~\cite{mironov2017} to compute the noise scale $\sigma$ required for a given privacy loss $\epsilon$, privacy failure probability $\delta$, number of parameter updates $M$, batch size $| \calB |$, and training set size $N$.

\subsection{Private Prediction}
\label{sec:private_prediction}

We consider two private prediction methods in this study: (1) the prediction sensitivity method and (2) the subsample-and-aggregate method.

\noindent\textbf{\underline{Prediction sensitivity.}}
The prediction sensitivity method adds noise to the logits, $\hat{\bl} = \phi'(\hat{\bx}; \theta)$, predicted by the model.
Specifically, it returns $\hat{\bl} + \bb$, where the noise vector $\bb \in \mathbb{R}^C$ is sampled from $p(\bb) \propto e^{-\beta \|\bb\|_2}$.
Given an inference budget $B$, we use standard composition~\cite{dwork2006} and set $\beta=\frac{N\lambda\epsilon}{2KB}$.

\begin{theorem}
  \label{thm:prediction_sensitivity}
  Given assumptions \ref{as:loss}, \ref{as:regularizer}, \ref{as:linear},
  \ref{as:lip}, and \ref{as:l2inp}, the prediction sensitivity method is
  $(\epsilon, 0)$-differentially private.
\end{theorem}

We can obtain $(\epsilon, \delta)$-differentially private predictions for $\delta > 0$ by sampling $\bb$ from a zero-mean isotropic Gaussian distribution~\cite{balle2018improving}.
Assuming $\delta > 0$ also allows the use of the advanced composition theorem~\cite{dwork2016concentrated,dwork2010boosting}.
We set the standard deviation to $\sigma = \frac{2K \alpha^*}{N\lambda\epsilon^*}$, where $\alpha^*$ depends on per-sample privacy values $\epsilon^*$ and $\delta^*$ that are found using a search algorithm (see appendix for details).

\begin{theorem}
  \label{thm:gaussian_prediction_sensitivity}
  Given assumptions \ref{as:loss}, \ref{as:regularizer}, \ref{as:linear},
  \ref{as:lip}, and \ref{as:l2inp}, the Gaussian prediction sensitivity method is
  $(\epsilon, \delta)$-differentially private for $\delta \in (0, 1)$.
\end{theorem}

When the budget $B$ is small, the $\sigma$ obtained with standard composition can be smaller than that obtained with advanced composition. In experiments, we always choose the smaller $\sigma$ of the two.

\begin{figure*}[!t]\centering
  \subfloat[\label{fig:2a}For privacy failure probability $\delta=0$.]{  \includegraphics[width=0.45\linewidth]{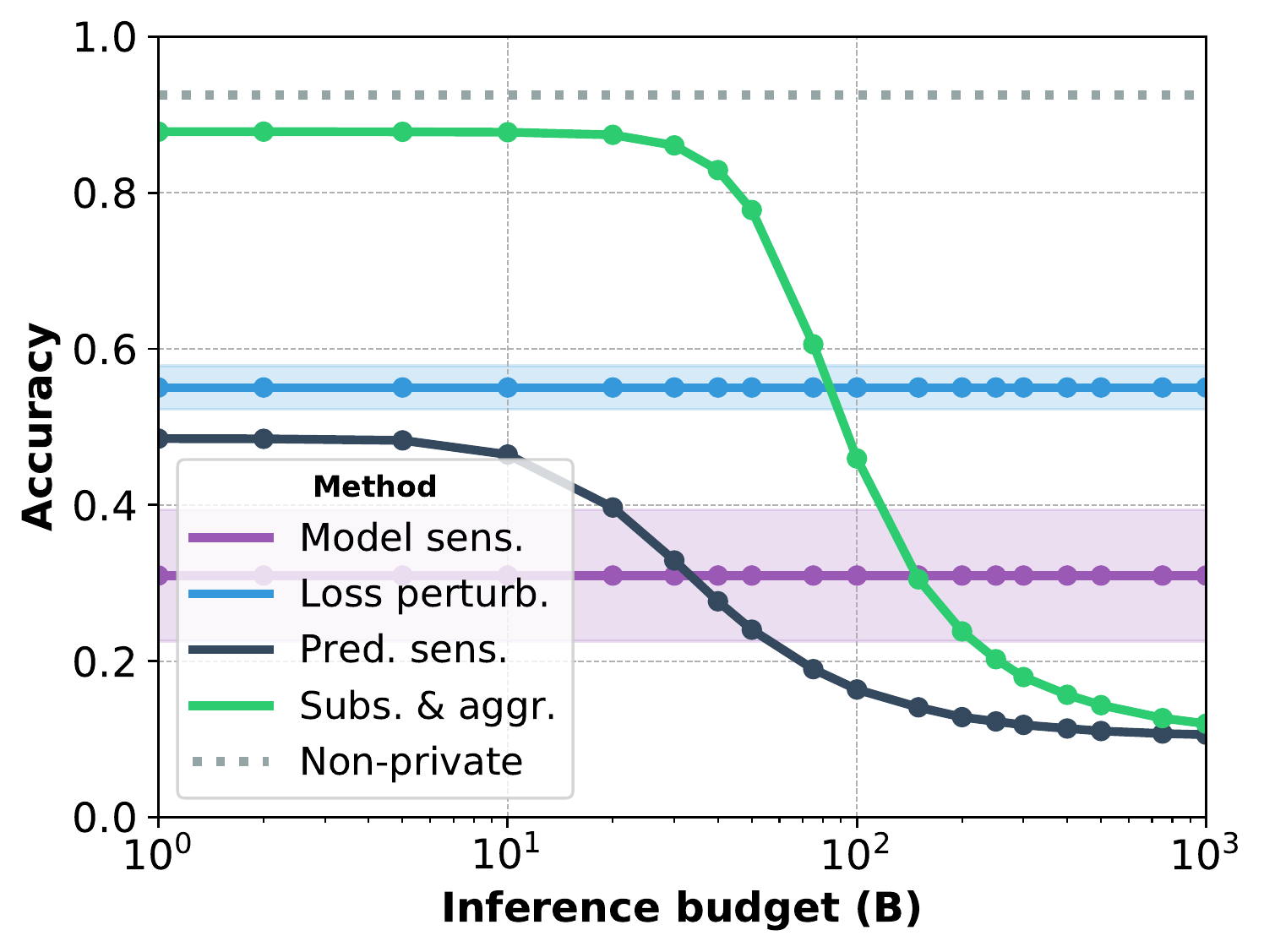}}
\hspace{6mm}
  \subfloat[\label{fig:2b}For privacy failure probability $\delta=10^{-5}$.]{  \includegraphics[width=0.45\linewidth]{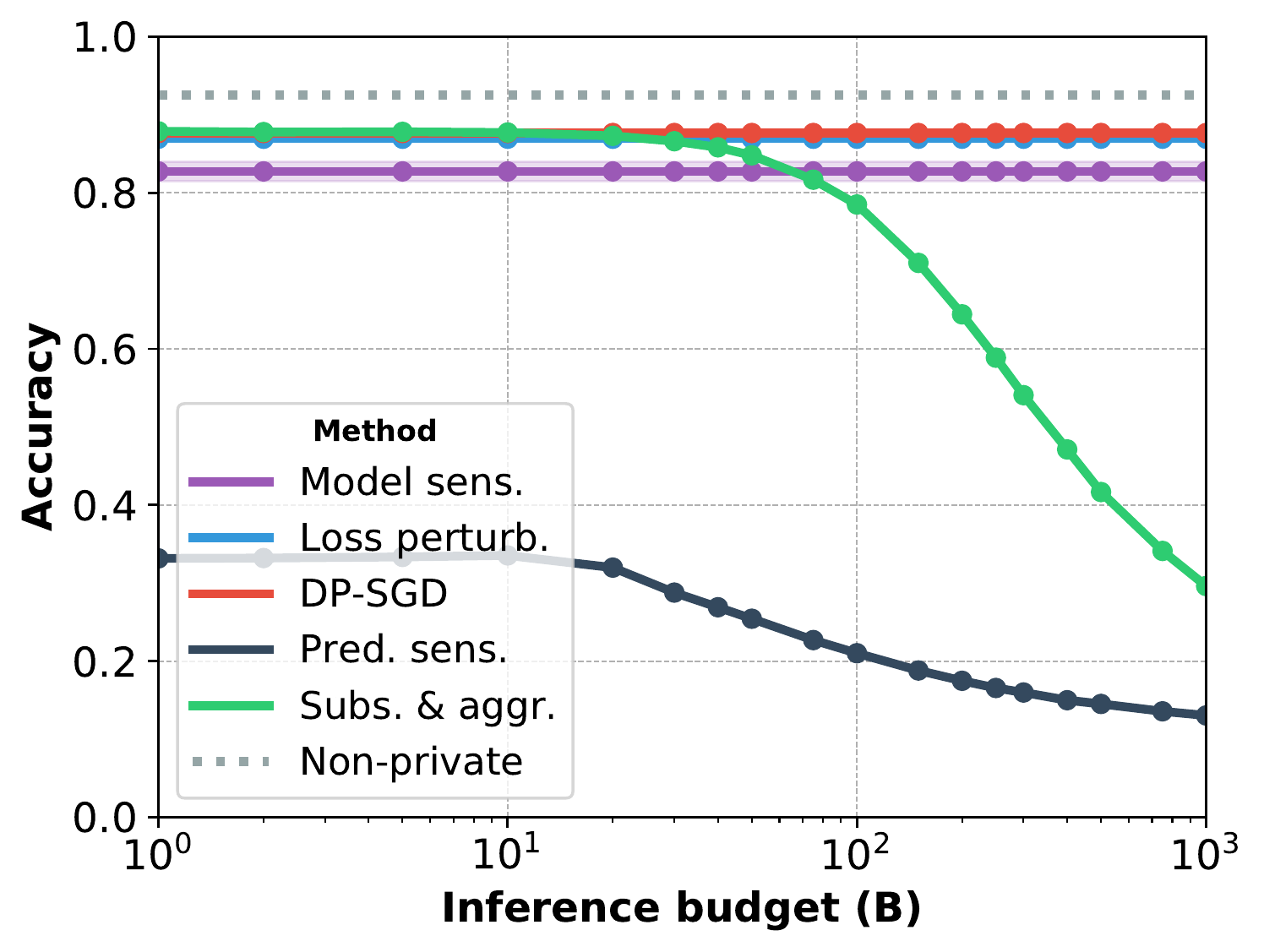}}
\caption{Test accuracy on MNIST dataset as function of inference budget $B$ for privacy loss $\epsilon=1$.}\
\label{fig:2}
\end{figure*}

\noindent\textbf{\underline{Subsample-and-aggregate.}}
Subsample-and-aggregate methods train $T$ models on $T$ subsets of $\calD$, and perform differentially private aggregation of the predictions produced by the resulting models~\cite{bassily2018,dwork2018,nandi2019,nissim07}.
We focus on the method of Dwork \& Feldman~\cite{dwork2018} in this study.
That technique: (1) partitions the training dataset $\calD$ into $T$ disjoint subsets of size $\lfloor \nicefrac{| \calD |}{T} \rfloor$, (2) trains $T$ classifiers $\{\phi'_1(\cdot; \theta_1), \dots, \phi'_T(\cdot; \theta_T) \}$ on these subsets where $\phi'_t(\cdot)$ outputs a one-hot vector of size $C$, and (3) applies a soft majority voting across the $T$ classifiers at inference time.
Hence, it predicts label $\hat{\by}$ for input $\hat{\bx}$ with probability proportional to $\exp(\beta \cdot | \{t : t \in \{1, \dots, T \}, \phi'_t(\hat{\bx}; \theta_t) = \hat{\by} \}  |)$.
The parameter $\beta$ acts as an ``inverse temperature'' in the voting: privacy increases but accuracy decreases as $\beta$ goes to zero.
Using the standard compositional properties of differential privacy~\cite{dwork2006}, this procedure achieves $(\epsilon, 0)$-differential privacy with an inference budget of $B$ by setting $\beta = \epsilon / B$.

\begin{theorem}
  \label{thm:subsample_and_aggregate}
  The subsample-and-aggregate method with $\beta = \epsilon / B$ is $(\epsilon, 0)$-differentially private.
\end{theorem}

If we allow $\delta > 0$, we can use the advanced
composition theorem~\cite{dwork2016concentrated, dwork2010boosting} in subsample-and-aggregate to achieve
$(\epsilon, \delta)$-differential privacy with noise scale $\beta$ inversely proportional
to the square of $B$.

\begin{theorem}
  \label{thm:delta_subsample_and_aggregate}
  The subsample-and-aggregate method with $\beta=\max(\beta', \beta'')$, where $\beta' = \epsilon / B$ and $\beta'' =  \sqrt{2 / B}\left(\sqrt{\ln(1/\delta) + \epsilon} - \sqrt{\ln(1/\delta)}\right)$, is $(\epsilon, \delta)$-differentially private.
\end{theorem}

\begin{figure*}[!t]
\begin{minipage}{0.465\linewidth}
\centering
\includegraphics[width=\linewidth]{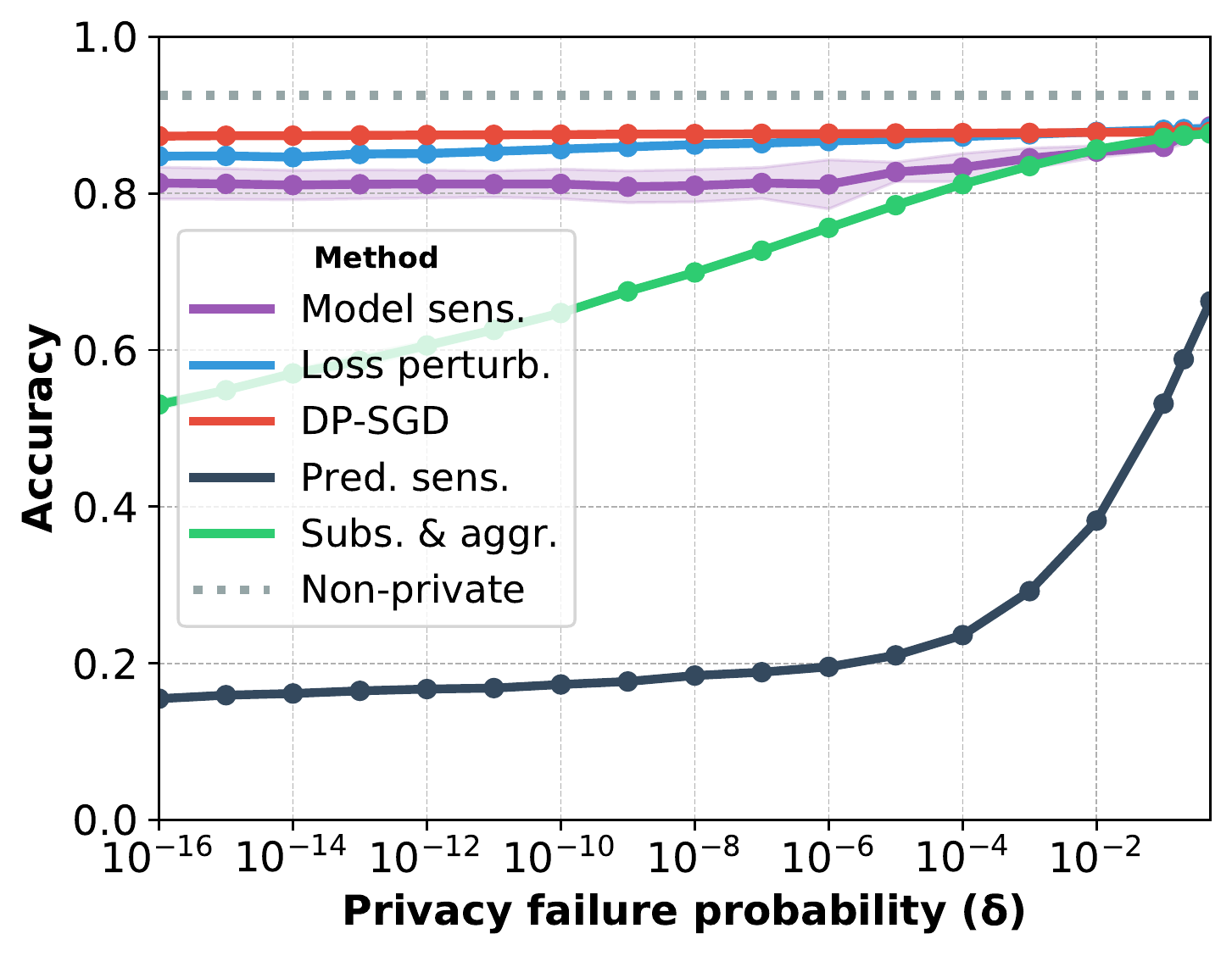}
\caption{Test accuracy on MNIST dataset as a function of privacy failure probability, $\delta$, for privacy $\epsilon=1$ and inference budget $B=100$.}
\label{fig:3}
\end{minipage}
\hspace{6mm}
\begin{minipage}{0.475\linewidth}
\centering
\includegraphics[width=\linewidth]{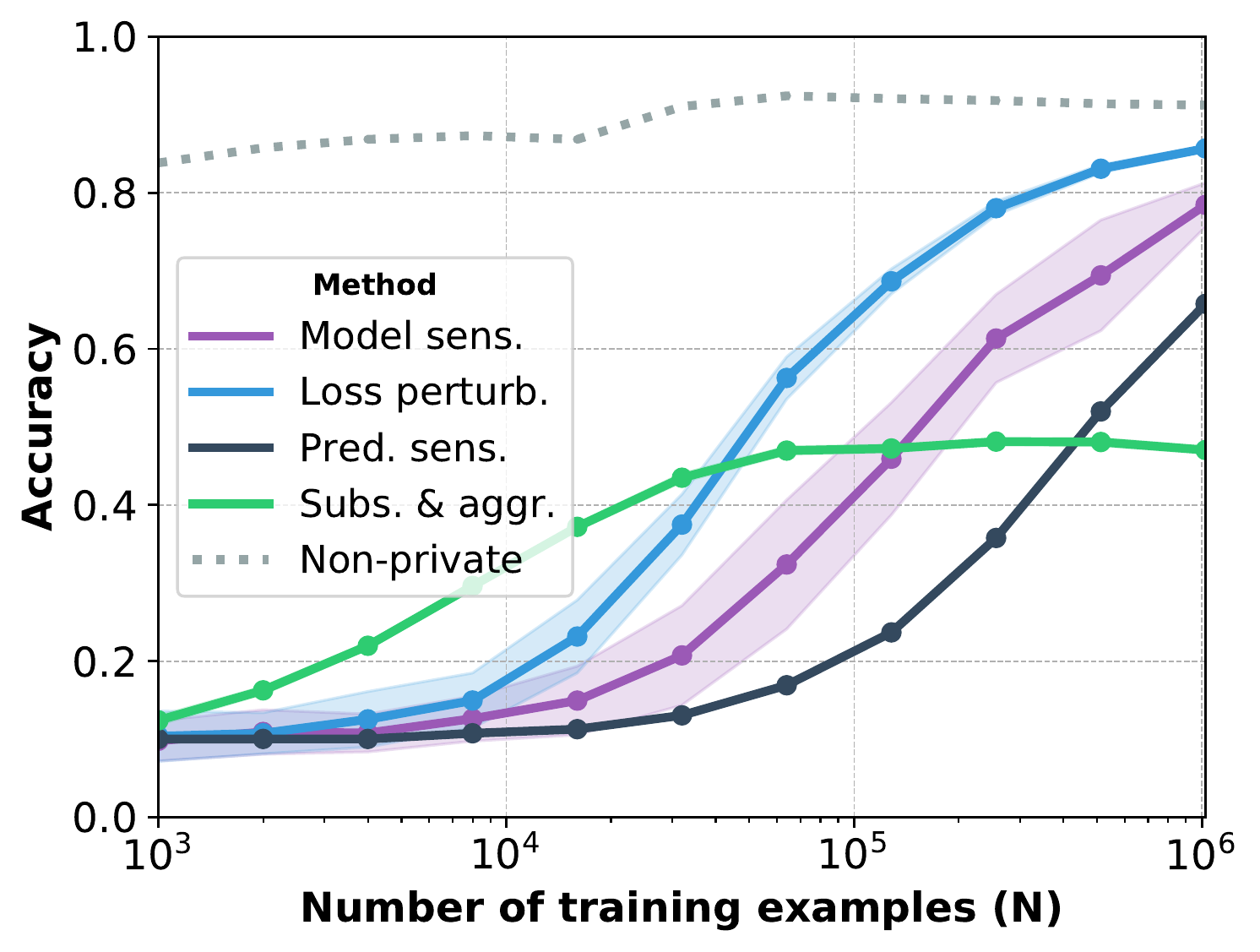}
\caption{Test accuracy on MNIST-1M dataset as a function of training set size, $N$, for privacy $(\epsilon, \delta) = (1, 0)$ and inference budget $B=100$.}
\label{fig:4}
\end{minipage}
\end{figure*}

\section{Experiments}
\label{sec:experiments}
We evaluate the five private prediction methods in experiments with linear models on the MNIST~\cite{lecun1998} dataset (\ref{sec:results_linear}) and convolutional networks on the CIFAR-10~\cite{krizhevsky2009} dataset (\ref{sec:results_convnet}).
Code reproducing our experiments is available from \url{https://github.com/facebookresearch/private_prediction}.

\subsection{Experimental Setup}
In all experiments, we choose the loss function $\ell$ to be the multi-class logistic loss.
When training linear models, we minimize the loss using L-BFGS~\cite{zhu1997} with a line search that checks the Wolfe conditions~\cite{wolfe1969} for all methods except the DP-SGD method.
We trained all models using $L_2$-regularization except when using the DP-SGD or subsample-and-aggregate methods.
(We note that these two methods have implicit regularization via gradient noise and ensembling, respectively.)
We selected regularization parameter $\lambda$ via cross-validation in a range between $10^{-5}$ and $5\cdot10^{-1}$.
We choose the clip value $\nu$ in DP-SGD by cross-validating over a range from $5 \cdot 10^{-4}$ to $5 \cdot 10^{-1}$.
For simplicity, we ignore the privacy leakage due to cross-validation and hyperparameter tuning in our analysis.
Unless noted otherwise, we perform subsample-and-aggregate with $T=256$ models.

To evaluate the quality of our models, we measure their classification accuracy on the test set.
Because private prediction methods are inherently noisy, we repeat every experiment $100$ times using the same train-test split.
We report the average accuracy corresponding standard deviation in all result plots.

\subsection{Results: Linear Models}
\label{sec:results_linear}

We first evaluate the performance of linear models on the MNIST handwritten digit dataset~\cite{lecun1998}.

\noindent{\textbf{Accuracy as a function of privacy loss.}}
Figure~\ref{fig:1} displays the accuracy of the models on the test set (higher is better) as a function of the privacy loss $\epsilon$ (lower values imply more privacy) for an inference budget of $B=100$.
The results are presented for two values of the privacy failure probability: $\delta = 0$ in Figure~\ref{fig:1a} and $\delta=10^{-5}$ in Figure~\ref{fig:1b} (DP-SGD does not support $\delta = 0$).
Recall that most methods support both $\delta = 0$ and $\delta > 0$ by choosing different noise distributions or a different privacy analysis.

Figure~\ref{fig:1} shows that the subsample-and-aggregate method outperforms the other methods for most values of $\epsilon$ in the $\delta=0$ regime, but is less competitive when $\delta=10^{-5}$ (higher probability of privacy failure).
Loss perturbation performs best for small values of $\epsilon$ (much privacy) when $\delta=0$.
Model sensitivity underperforms the other two methods when $\delta=0$ but is slightly more competitive when $\delta=10^{-5}$ due to the use of a Gaussian noise distribution.
DP-SGD outperforms all other methods when $\delta=10^{-5}$.
Prediction sensitivity performs poorly across the board.

\begin{figure*}[!t]\centering
  \subfloat[\label{fig:5a}As a function of number of dimensions, $D$.]{  \includegraphics[width=0.46\linewidth]{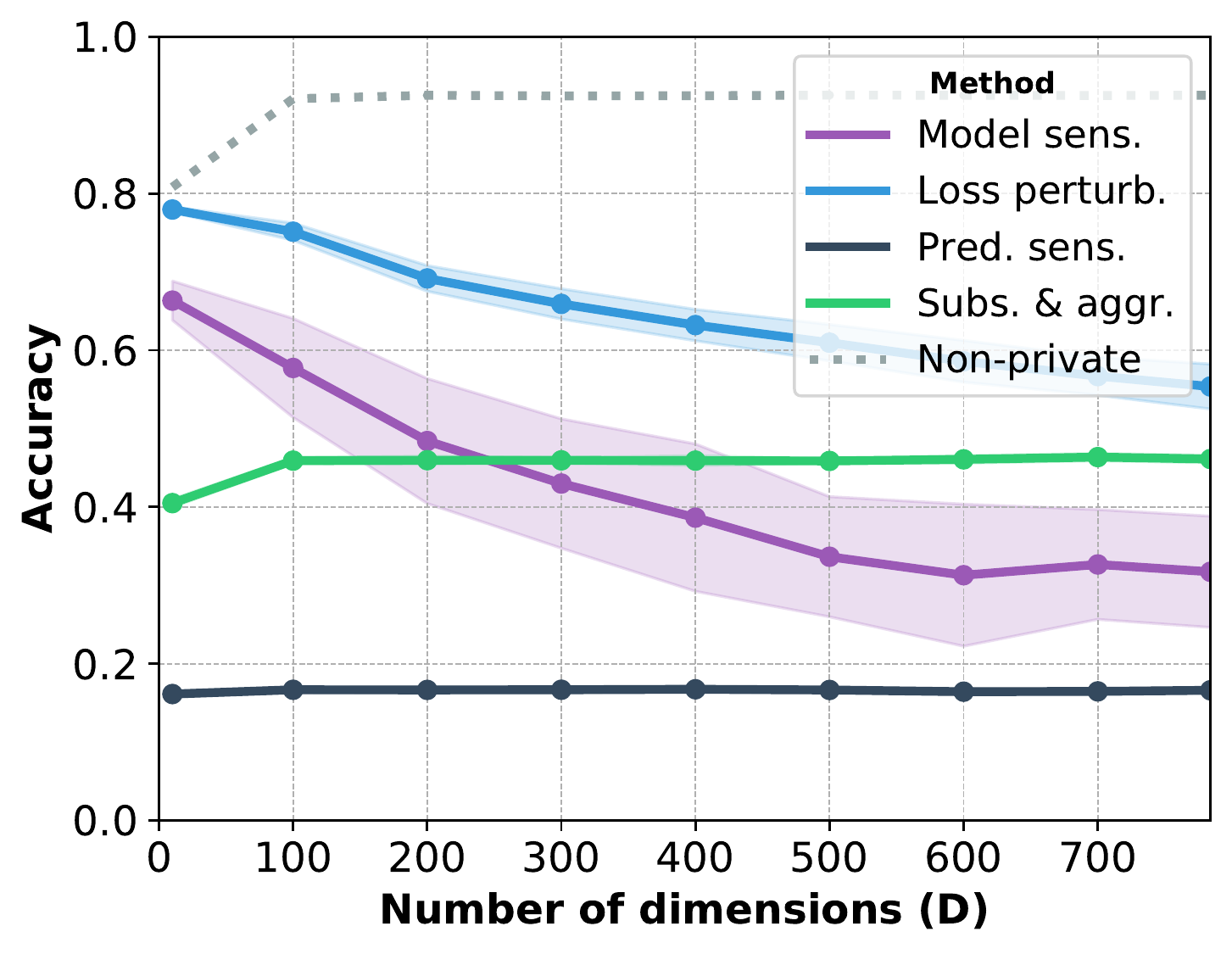}}
\hspace{6mm}
  \subfloat[\label{fig:5b}As a function of number of classes, $C$.]{  \includegraphics[width=0.46\linewidth]{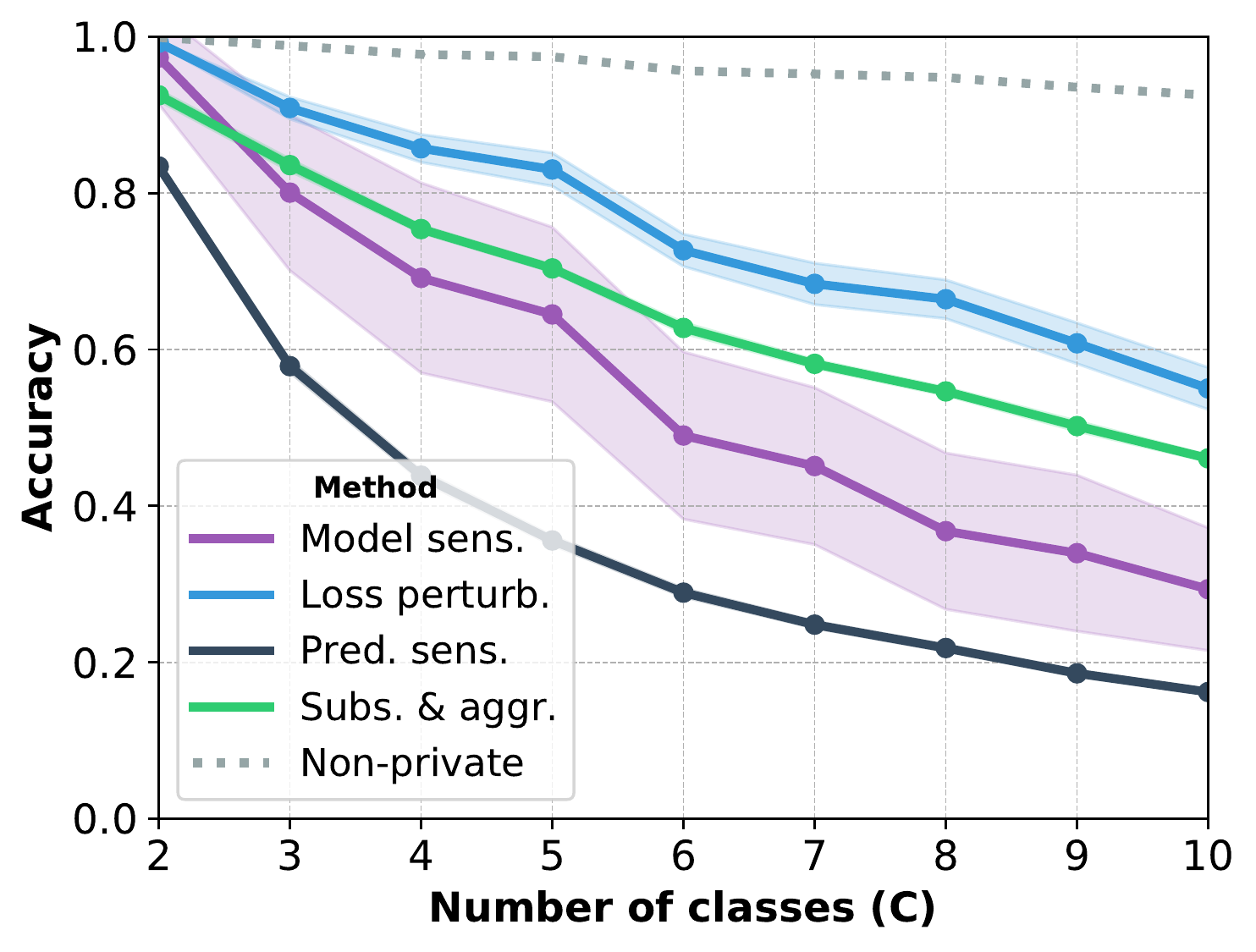}}
\caption{Test accuracy on MNIST dataset for privacy $(\epsilon, \delta) = (1, 0)$ and inference budget $B=100$.}
\label{fig:5}
\end{figure*}

\noindent{\textbf{Accuracy as a function of inference budget.}}
Figure~\ref{fig:1} shows experimental results for a particular choice of inference budget ($B=100$).
In private prediction methods, the amount of privacy provided changes when this budget changes.
To investigate this trade-off, Figure~\ref{fig:2} displays accuracy as a function of the inference budget, $B$, for a privacy loss setting of $\epsilon = 1$.
As before, we show results for privacy failure probability $\delta = 0$ (Figure~\ref{fig:2a}) and $\delta=10^{-5}$ (Figure~\ref{fig:2b}).
In the plots, private training methods are horizontal lines because the amount of information they leak does not depend on $B$.

The results show that the subsample-and-aggregate method outperforms other private prediction methods in the low-budget regime when $\delta=0$.
However, this advantage vanishes when the number of inferences that needs to be supported, $B$, grows.
Specifically, loss perturbation outperforms subsample-and-aggregate at $B \gtrapprox 80$ predictions.
The advantage also disappears when a small privacy failure probability is allowed.
For $\delta=10^{-5}$, two private training methods (loss perturbation and DP-SGD) perform at least as good as subsample-and-aggregate for all inference-budget values.
We also note that both loss perturbation and model sensitivity benefit greatly from making $\delta$ non-zero.

\noindent{\textbf{Accuracy as a function of privacy failure probability.}}
Figure~\ref{fig:3} studies the effect of varying the privacy failure probability, $\delta$, for a privacy loss of $\epsilon=1$ and inference budget $B=100$.
The results show that loss perturbation, DP-SGD, and model sensitivity work well even at very small values of $\delta$.
By contrast, prediction sensitivity appears to require unacceptably high values of $\delta$ to obtain acceptable accuracy.
The subsample-and-aggregate method benefits the most from increasing $\delta$.

\noindent{\textbf{Accuracy as a function of training set size.}}
Accuracy and privacy are also influenced by the number of training examples, $N$.
Figure~\ref{fig:4} studies this influence for privacy $(\epsilon, \delta) = (1, 0)$ and inference budget $B=100$ on a dataset of up to one million digit images that were generated using InfiMNIST~\cite{loosli2006}. 
(We do not include DP-SGD because it does not support $\delta=0$.)
The results show that subsample-and-aggregate is more suitable for settings in which few training examples are available (small $N$) because of the regularizing effect of averaging predictions over $T$ models.
However, it benefits less from increasing the training set size when $T$ remains fixed (recall that $T=256$).
When $N$ is large, loss perturbation and model sensitivity perform well because the scale of the noise that needs to be added to obtain a certain value of $\epsilon$ decreases as $N$ increases.

\noindent{\textbf{Accuracy as a function of number of dimensions.}}
Accuracy and privacy may also vary with model capacity, which in linear models can be influenced via the number of input dimensions, $D$.
Figure~\ref{fig:5a} presents the results of experiments in which we varied $D$ by applying PCA on the digit images.
The figure shows that loss perturbation and model sensitivity are more competitive when the model has low capacity to memorize examples ($D$ is small), because this implies less parameter noise is needed.

\noindent{\textbf{Accuracy as a function of number of classes.}}
Figure~\ref{fig:5b} investigates accuracy as a function of the number of classes, $C$.
In this experiment, a digit classification problem with $C' < C$ classes only considers the first $C'$ digits in the MNIST dataset (digits $0$ to $C'-1$).
Unsurprisingly, the accuracy of all methods decreases as the number of classes increases due to the increase in the problem's Bayes error.
The relative ranking of the methods appears to be largely independent of the number of classes.

\noindent{\textbf{Accuracy as a function of hyperparameters.}}
Figure~\ref{fig:6} studies the effect of: (1) the $L_2$-regularization parameter, $\lambda$, in loss perturbation and (2) the number of models, $T$, in subsample-and-aggregate.
The results show that models overfit to the perturbation when $\lambda$ is too small, but underfit when $\lambda$ is too large.
The optimal value for $\lambda$ depends on $\epsilon$.
Similarly, larger values of $T$ lead to higher accuracies but accuracy deteriorates for very large values of $T$ due to overfitting.
We note that increasing $T$ also increases computational requirements for training and inference, which grow as $O(T)$.

\begin{figure*}[!t]\centering
  \subfloat[\label{fig:6a}Loss perturbation: As a function of $\lambda$.]{  \includegraphics[width=0.45\linewidth]{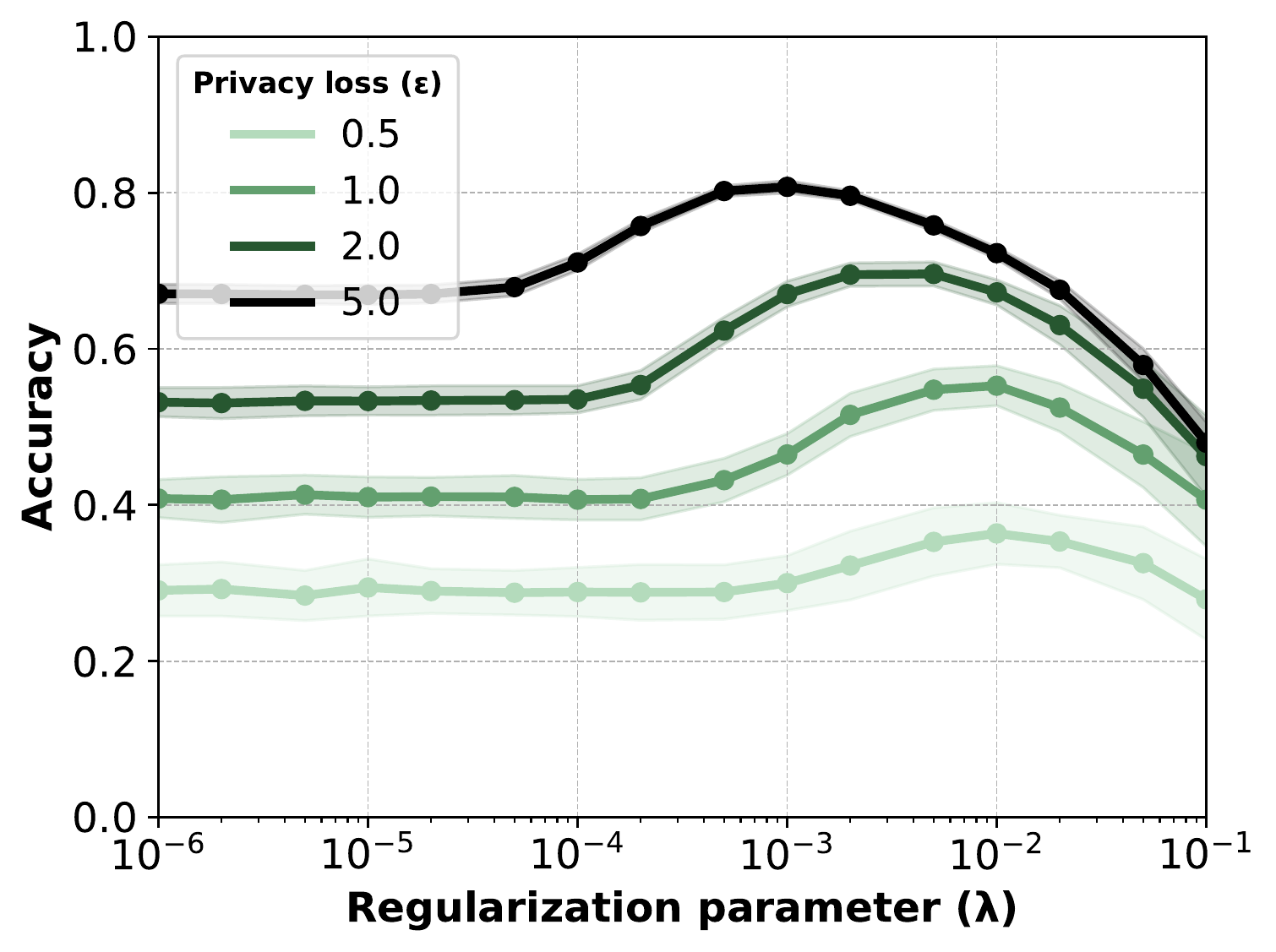}}
\hspace{6mm}
  \subfloat[\label{fig:6b}Subsample-and-aggregate: As a function of $T$.]{  \includegraphics[width=0.45\linewidth]{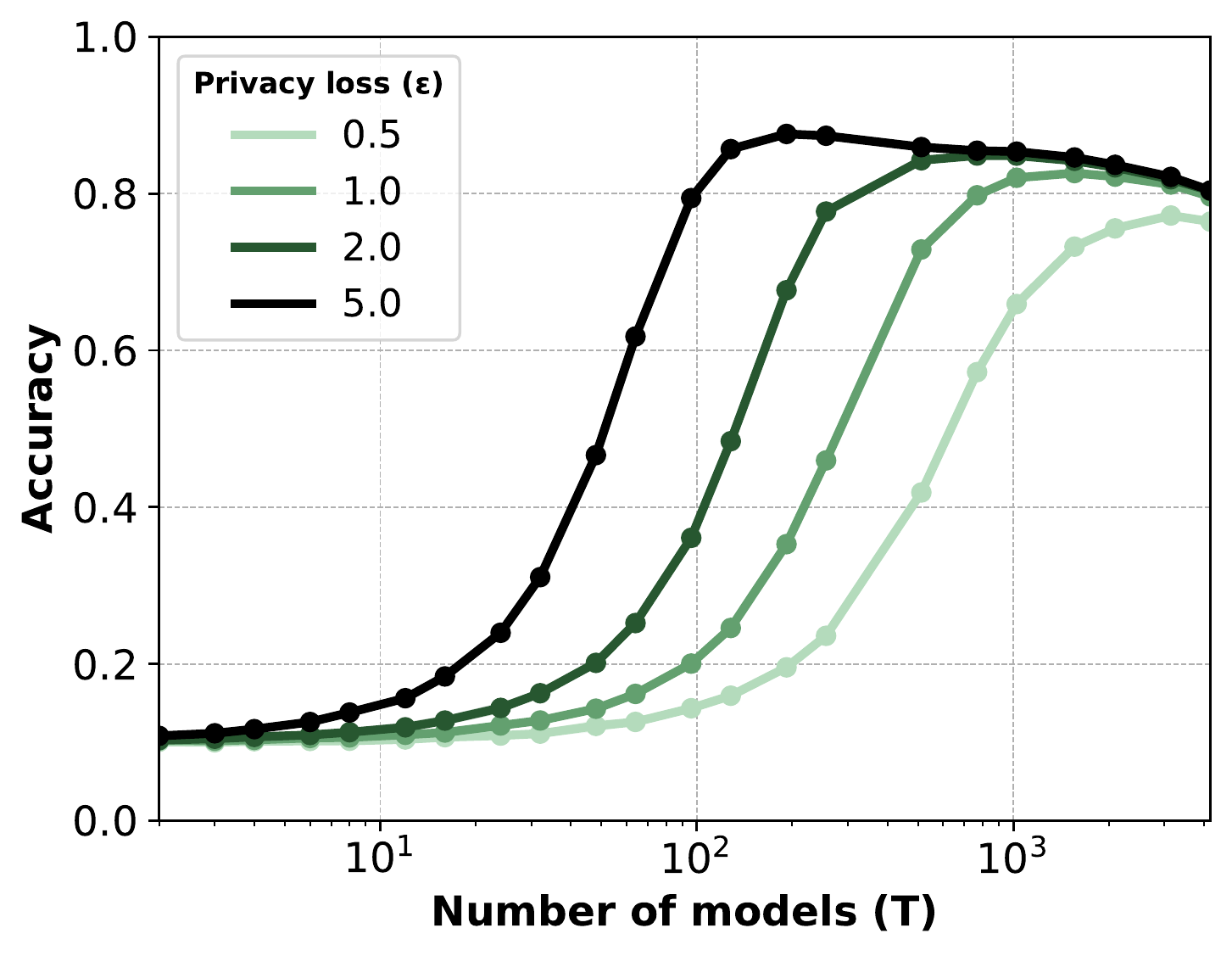}}
\caption{Test accuracy on MNIST dataset for privacy failure probability $\delta = 0$ and budget $B=100$.}
\label{fig:6}
\end{figure*}

\subsection{Results: Convolutional Networks}
\label{sec:results_convnet}
We also evaluated the DP-SGD and subsample-and-aggregate methods\footnote{Other private prediction methods cannot be used here because deep networks violate assumptions \ref{as:loss} and \ref{as:linear}.} on the CIFAR-10 dataset~\cite{krizhevsky2009} using a ResNet-20 model (with ``type A'' blocks~\cite{he2016deep}) as $\phi'(\cdot)$.
To facilitate the computation of per-example gradients in DP-SGD, we replaced batch normalization by group normalization~\cite{wu2018groupnorm}.
A non-private version of the resulting model achieves a test accuracy of $86.6\%$ on CIFAR-10.
All convolutional networks are trained using SGD with a batch size of 600, with an initial learning rate of $0.1$.
In all experiments, the learning rate was divided by $10$ thrice after equally spaced numbers of epochs.
We use standard data augmentation during training: \emph{viz.}, random horizontal flipping and random crop resizing.
The $T = 16$ models used in subsample-and-aggregate were trained for 500 epochs.
Because privacy loss increases with number of epochs in DP-SGD, we trained our DP-SGD models for only 100 epochs.
As before, we set the clip value $\nu$ in DP-SGD via cross-validation.

Figure~\ref{fig:7} shows the accuracy of the resulting models as a function of privacy loss, $\epsilon$, for inference budget $B=10$ (Figure~\ref{fig:7a}); and as a function of inference budget, $B$, for privacy loss $\epsilon=1$ (Figure~\ref{fig:7b}).
All results are for a privacy failure probability of $\delta=10^{-5}$.
The results in the figure show that the subsample-and-aggregate method outperforms DP-SGD training in some regimes: in particular, when $B$ is small and $\epsilon$ is large.
In most other situations, however, DP-SGD appears to be the best method to obtain private predictions from a ResNet-20: for privacy loss $\epsilon=1$, it surpasses subsample-and-aggregate for $B \gtrapprox 5$ predictions.
In terms of accuracy, both methods lead to substantial losses in accuracy compared to a non-private model: even at a relatively large privacy loss of $\epsilon=5$, the best private model underperforms its non-private counterpart by at least $30\%$.

\begin{table}[b]
\vspace{-1.5mm}
\captionof{table}{Effect of six hyperparameters (columns) on privacy, accuracy, and computational properties (rows): $\nearrow$ indicates that a property value goes up with the hyperparameter, $\searrow$ that it goes down, $\curvearrowright$ that it goes up and then down, and $\cdots$ that it remains unchanged. Effects for method-independent hyperparameters are in the left columns: the number of training examples $N$, the number of classes $C$, and the number of dimensions $D$. Effects for method-specific hyperparameters are in the right columns: the inference budget $B$, the number of models $T$, and the $L_2$-regularization parameter $\lambda$.}
\label{table:variation_summary}
\centering
\begin{tabular}{ l c c c c c c c c}
\toprule
  & $\mathbf{N}$ & $\mathbf{C}$ & $\mathbf{D}$ &~& $\mathbf{B}$ & $\mathbf{T}$ & $\boldsymbol{\lambda}$  \\
\midrule
\rowcolor{Gray} Privacy loss ($\epsilon$)  & $\searrow$ & $\nearrow$ & $\nearrow$ &~& $\nearrow$ & $\searrow$         & $\searrow$ \\
Privacy failure probability ($\delta$)     & $\searrow$ & $\cdots$   & $\cdots$   &~& $\nearrow$ & $\cdots$           & $\searrow$ \\
\rowcolor{Gray} Accuracy                   & $\nearrow$ & $\searrow$ & $\nearrow$ &~& $\searrow$ & $\curvearrowright$ & $\curvearrowright$ \\
Training time                              & $\nearrow$ & $\nearrow$ & $\nearrow$ &~& $\cdots$   & $\nearrow$         & $\cdots$ \\
\rowcolor{Gray} Inference time             & $\cdots$   & $\nearrow$ & $\nearrow$ &~& $\cdots$   & $\nearrow$         & $\cdots$ \\
\bottomrule
\end{tabular}
\end{table}

\section{Discussion}
\label{sec:discussion}
The results of our experiments demonstrate that there exist a range of complex trade-offs in private prediction.
We summarize these trade-offs in Table~\ref{table:variation_summary}.
The table provides an overview of how privacy, accuracy, and computation are affected by increases in variables such as the number of training examples, the inference budget, the number of input dimensions, the amount of regularization, \emph{etc.}
We hope that this overview provides practitioners some guidance on the trade-offs of private prediction.

\begin{figure*}[!t]\centering
  \subfloat[\label{fig:7a}As a function of privacy loss, $\epsilon$, for $B=10$.]{  \includegraphics[width=0.45\linewidth]{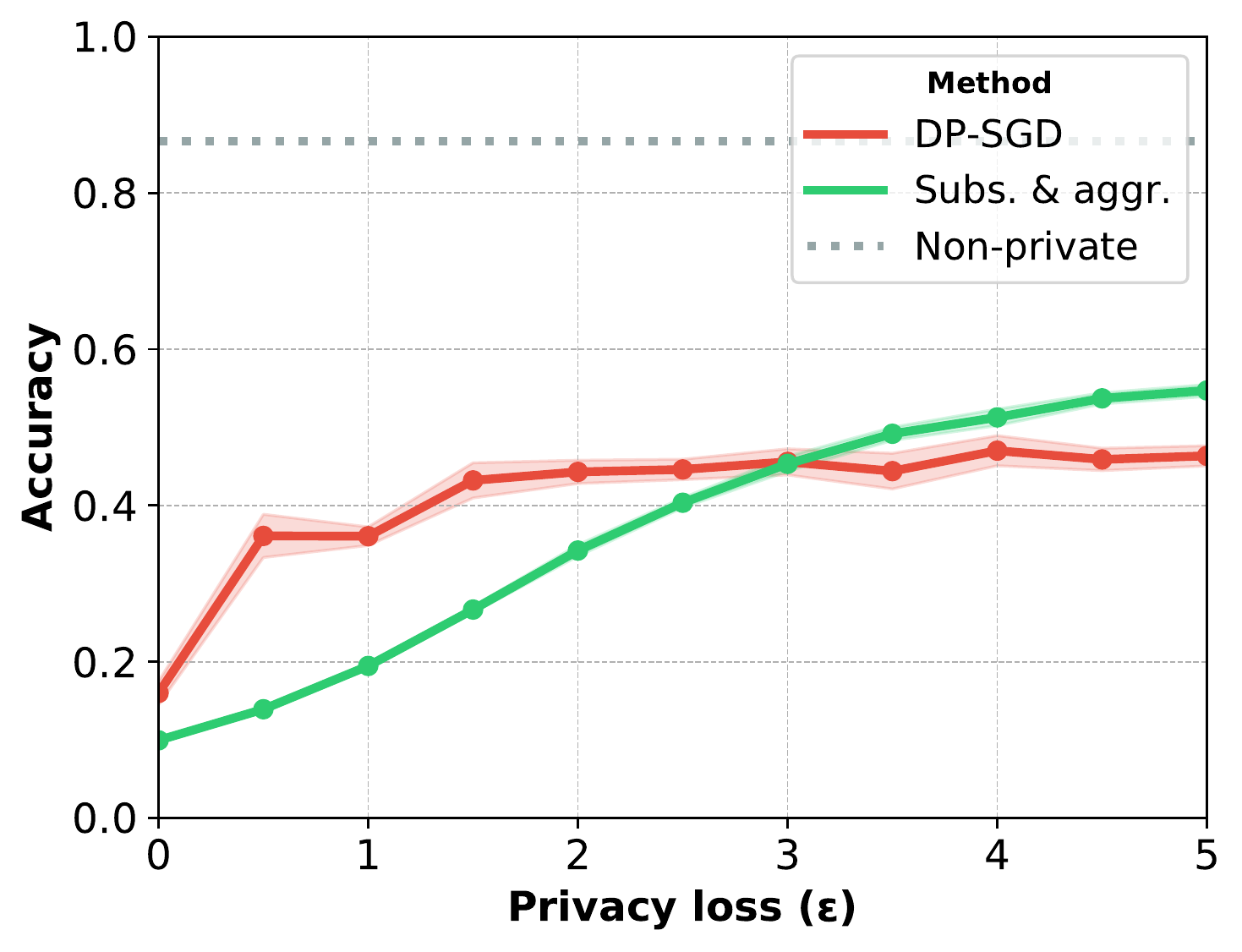}}
\hspace{6mm}
  \subfloat[\label{fig:7b}As a function of inference budget, $B$, for $\epsilon=1$.]{  \includegraphics[width=0.46\linewidth]{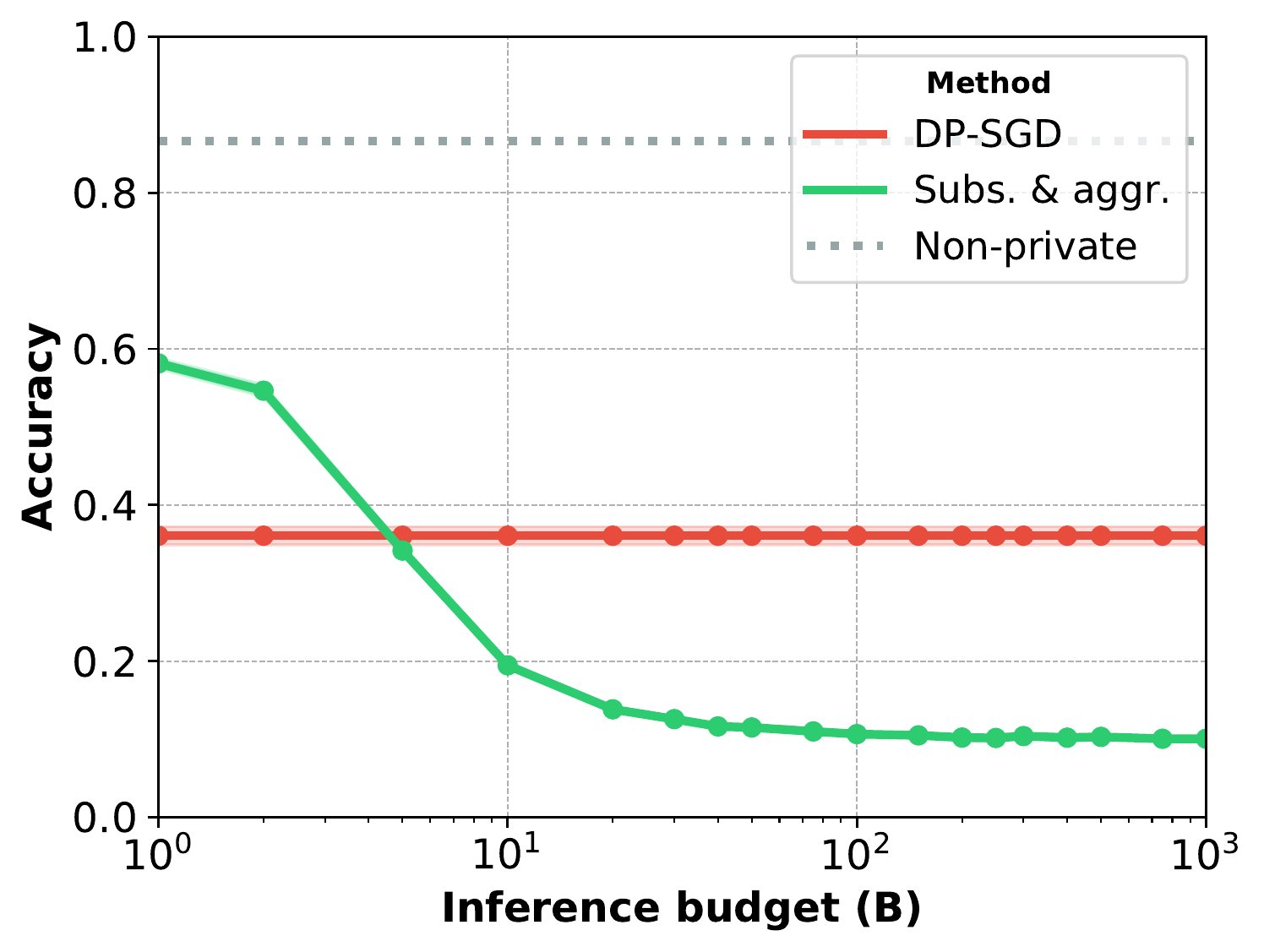}}
\caption{Test accuracy on CIFAR-10 dataset of ResNet-20 with group normalization for $\delta=10^{-5}$.}
\label{fig:7}
\end{figure*}

A perhaps surprising result of our study is that private training methods are often a better alternative for private prediction than private prediction methods.
Methods such as subsample-and-aggregate appear to be most competitive for small inference budgets (tens or hundreds of test examples) and for small training sets, but tend to be surpassed by private training methods otherwise.
We surmise that the main reason for this observation is that inference-budget accounting is fairly naive: the privacy guarantees assume that privacy loss grows linearly ($\delta=0$) or with the square root ($\delta > 0$) of the number of inferences, which may be overly pessimistic.
Inference-budget accounting may be improved, \emph{e.g.}, using accounting mechanisms akin to those in the PATE family of algorithms~\cite{papernot2016,papernot2018}.

Another way in which private prediction methods may be improved is by making stronger assumptions about the models with which they are used.
In particular, the subsample-and-aggregate method is the only method that treats models as a pure black box.
This is an advantage in the sense that it provides flexibility compared to private training methods that support only a small family of models: for example, subsample-and-aggregate could also be used to make private predictions for decision trees.
But the flexibility of subsample-and-aggregate is also a weakness because it may cause the privacy guarantees to be too pessimistic.
Indeed, it may be possible to improve private prediction methods for specific models by making additional assumptions about those models (linearity, smoothness, \emph{etc.}) or by adapting the models to make them more suitable for private training or prediction~\cite{papernot2019}.

\section*{Broader Impact}
To date, there are a large number of machine-learning systems that are trained on data that needs to remain private or confidential but that expose predictions to the outside world. 
Examples of such systems include cloud services, content-ranking systems, and recommendation services. 
In controlled settings, experiments have shown that it may be possible to extract private information from such predictions. 
Yet, the machine-learning community currently has limited understanding of how much private information can be leaked from such services and how this information leakage can be limited. 
Over the past decade, a range of methods has been developed that aim to limit information leakage but these methods have primarily been studied theoretically.
Indeed, little is known about their empirical performance in terms of the inevitable privacy-accuracy trade-off. 
This work helps to improve that understanding.
As such, we believe this study helps anyone whose data is used to train modern machine-learning systems.
In particular, our study can help the developers of those systems to better understand the privacy risks in their systems and may also help them find better operating points on the privacy-accuracy spectrum.

\section*{Acknowledgements}
We thank Ilya Mironov, Davide Testuginne, Mark Tygert, Mike Rabbat, and Aaron Roth for helpful feedback on early versions of this paper.

\bibliographystyle{abbrv}

\newpage
\appendix

\section{Preliminaries}

As a reminder, we denote by $\phi(\bx; \theta)$ a machine learning model with parameters $\theta$ that given a $D$-dimensional input vector, $\bx \in \mathbb{R}^D$, produces a probability vector over $C$ classes $\by \in \Delta^C$, where $\Delta^C$ represents the $(C\!-\!1)$-dimensional probability simplex.
The parameters $\theta$ are obtained by fitting the model on a training set of $N$ labeled examples, $\calD = \{(\bx_1, \by_1), \dots, (\bx_N, \by_N) \}$.

In the following we aim to bound the privacy loss on the training set $\calD$ given a budget of $B$ queries, $\mathcal{Q} =\{\hat{\bx}_1, \dots, \hat{\bx}_B\}$, to the model $\phi(\cdot)$.

We adopt a regularized empirical risk minimization framework in which we minimize:
\begin{equation}
    \label{appxeq:erm}
    J(\theta; \calD) = \frac{1}{N} \sum_{n=1}^N \ell(\phi'(\bx_n; \theta), \by_n)  + \lambda R(\theta).
\end{equation}
with respect to $\theta$, where $\ell(\cdot)$ is a loss function, $R(\cdot)$ is a regularizer, and $\lambda \geq 0$ is a regularization parameter.
Some of the methods make one or more of the following assumptions:

\begin{enumerate}[leftmargin=*]
\setlength\itemsep{0em}
    \item The loss function $\ell(\cdot)$ is strictly convex, continuous, and
        differentiable everywhere.\label{appxas:loss}
    \item The regularizer $R(\theta)$ is $1$-strongly convex, continuous, and
        differentiable everywhere w.r.t. $\theta$. \label{appxas:regularizer}
    \item The model $\phi'(\cdot)$ is linear, \emph{i.e.}, $\phi'(\bx; \theta) = \theta^\top \bx$. \label{appxas:linear}
    \item The loss function $\ell(\cdot)$ is Lipschitz with a constant $K$, \emph{i.e.}, $\|
        \nabla l\|_2 \le K$. \label{appxas:lip}
      \item The inputs are contained in the unit $L_2$ ball, \emph{i.e.} $\|\bx\|_2 \le 1$
        for all $\bx$. \label{appxas:l2inp}
\end{enumerate}

In the following, we denote by $\|\bA\|_F$ the Frobenius norm of the matrix $\bA$, \emph{i.e.}:
\begin{equation}
    \|\bA\|_F = \left(\sum_{i,j} \bA_{ij}^2\right)^{\nicefrac{1}{2}}.
\end{equation}

We use Lemma~\ref{lem:sensitivity_bound} in the privacy proofs for the model
and prediction sensitivity methods.

\begin{lemma}
  \label{lem:sensitivity_bound}
  Given assumptions \ref{appxas:loss}, \ref{appxas:regularizer}, \ref{appxas:linear},
  \ref{appxas:lip}, and \ref{appxas:l2inp}, and letting $\theta_1\!=\!\argmin_{\theta}
  J(\theta; \calD)$ and $\theta_2\!=\!\argmin_{\theta} J(\theta; \calD')$, the
  sensitivity is bounded as $\|\theta_1 - \theta_2 \|_F \le
  \frac{2K}{N\lambda}$.
\end{lemma}

\begin{proof}
    To bound the sensitivity, let $G(\theta) = J(\theta; \calD)$ and $g(\theta) = J(\theta; \calD') -
    J(\theta; \calD)$. Note that $\theta_1 = \argmin_\theta G(\theta)$ and $\theta_2 = \argmin_\theta
    G(\theta) + g(\theta)$. Given Assumptions~\ref{appxas:loss} and~\ref{appxas:regularizer},
    we can apply Lemma 7 of~\cite{chaudhuri2011} to get:
    \begin{equation}
        \label{eq:lemma7}
        \|\theta_1 - \theta_2\|_F \le \frac{1}{\lambda} \max_\theta \|\nabla g(\theta)\|_F.
    \end{equation}

    Next we bound $\|\nabla g(\theta)\|_F$. Let $\bx,\by$ and $\bx',\by'$ be the
    differing examples in $\calD$ and $\calD'$. Using assumption
    ~\ref{appxas:linear}, we have:
    \begin{equation}
        g(\theta) = \frac{1}{N} (\ell(\theta^\top \bx, \by) - \ell(\theta^\top \bx', \by')).
    \end{equation}
    Thus:
    \begin{equation}
        \nabla g(\theta) = \frac{1}{N} (\nabla \ell(\theta^\top \bx, \by) \bx^\top - \nabla \ell(\theta^\top \bx', \by') \bx'^\top),
    \end{equation}
    and:
    \begin{align*}
        &\| \nabla g(\theta)\|_F \\
        &= \frac{1}{N} \left\|\nabla \ell(\theta^\top \bx, \by) \bx^\top - \nabla \ell(\theta^\top \bx', \by') \bx'^\top \right\|_F  \\
        &\le \frac{1}{N} \left( \left\|\nabla \ell(\theta^\top \bx, \by) \bx^\top \right\|_F + \left\| \nabla \ell(\theta^\top \bx', \by') \bx'^\top \right\|_F \right) \\
        &\le \frac{1}{N} \left( \left\|\nabla \ell(\theta^\top \bx, \by) \right\|_2 \left\|\bx \right\|_2 + \left\| \nabla \ell(\theta^\top \bx', \by') \right\|_2 \left\|\bx' \right\|_2 \right)\\
        &\le \frac{1}{N} K \left(\left\| \bx \right\|_2 + \left\| \bx' \right\|_2 \right) \\
        &\le \frac{2K}{N}.
    \end{align*}
    We use the triangle inequality in the third step, the fact that
    $\|{\bf u}{\bf v}^\top\|_F = \|{\bf u}\|_2 \|{\bf v}\|_2$ in the fourth step,
    and assumptions~\ref{appxas:lip} and~\ref{appxas:l2inp} in steps five and six
    respectively. Thus we have:
    \begin{equation}
        \label{eq:gradbound}
        \| \nabla g(\theta)\|_F \le \frac{2K}{N}.
    \end{equation}
    Combining Equations~\ref{eq:lemma7} and \ref{eq:gradbound} yields:
    \begin{equation}
      \label{eq:sensitivity_bound}
      \|\theta_1 - \theta_2 \|_F \le \frac{2K}{N\lambda},
    \end{equation}
    completing the proof.
\end{proof}

For the perturbations used by the sensitivity methods, we sample noise from the
distribution:
\begin{equation}
    \label{eq:sqrt_noise}
    p(\bB) = \frac{1}{z} e^{-\beta \|\bB\|_F},
\end{equation}
where $z$ is a normalizing constant. When $\delta > 0$ we use a zero-mean
isotropic Gaussian distribution with a standard deviation of $\sigma$:
\begin{equation}
    \label{eq:gaussian_noise}
    p(\bB) = \frac{1}{z} e^{-\frac{1}{2\sigma^2} \|\bB\|^2_F},
\end{equation}
where $z$ is a normalizing constant.

\section{Private Training}

\subsection{Model Sensitivity}

We generalize the sensitivity method of~\cite{chaudhuri2011} to the setting of
multi-class classification. The multi-class model sensitivity method achieves
$\epsilon$-differential privacy with respect to the dataset $\calD$ when
releasing the model parameters $\theta$. In this case, the budget of test
examples is infinite since the model itself is differentially private. The
model sensitivity method is given in Algorithm~\ref{alg:model_perturb}.
\begin{algorithm}
    \caption{Multi-class model sensitivity.}
    \label{alg:model_perturb}
\begin{algorithmic}
    \STATE {\bf Inputs:} Privacy parameter $\epsilon$, regularization parameter $\lambda$, and dataset $\calD$.
    \STATE {\bf Output:} $\theta_{\textrm{priv}}$, the $\epsilon$-differentially private classifier.
    \STATE Sample $\bB$ from the distribution in Equation~\ref{eq:sqrt_noise} with $\beta = \frac{N\lambda\epsilon}{2K}$.
    \STATE Compute $\theta_{\textrm{priv}} = \argmin_\theta J(\theta; \calD) + \bB$.
\end{algorithmic}
\end{algorithm}

\begin{appendix_theorem}
  \label{apxthm:model_perturb}
  Given assumptions \ref{appxas:loss}, \ref{appxas:regularizer}, \ref{appxas:linear},
  \ref{appxas:lip}, and \ref{appxas:l2inp}, the model sensitivity method in
  Algorithm~\ref{alg:model_perturb} is $\epsilon$-differentially private.
\end{appendix_theorem}

\begin{proof}
    We bound the privacy loss in terms of the sensitivity of the minimizer of
    $J(\cdot)$ (Equation~\ref{appxeq:erm}) in the Frobenius norm and then apply
    Lemma~\ref{lem:sensitivity_bound}. We denote by $\calA$ the application of
    differential privacy mechanism, \emph{i.e.} a run of
    Algorithm~\ref{alg:model_perturb}.

    For all datasets $\calD$ and $\calD'$ which differ by one example, we have:
    \begin{align*}
        &\log \frac{p(\calA(\calD) = \theta)}{p(\calA(\calD') = \theta)} \\
        &=\log \frac{e^{-\beta \|\theta_1 + \bB \|_F}}{e^{-\beta \|\theta_2 + \bB\|_F}} \\
        &= \beta \left(\|\theta_2 + \bB\|_F - \|\theta_1 + \bB \|_F \right)  \\
        &\le \beta \|\theta_1 - \theta_2 \|_F,
    \end{align*}
    where we use the triangle inequality, and the fact that:
    $\theta_1 = \argmin_\theta J(\theta; \calD)$ and $\theta_2 = \argmin_\theta J(\theta; \calD')$. Thus:
    \begin{equation}
        \log \frac{p(\calA(\calD) = \theta)}{p(\calA(\calD') = \theta)} \le
        \beta \|\theta_1 - \theta_2 \|_F. \label{eq:privloss}
    \end{equation}

    Combining Equation~\ref{eq:privloss} with Lemma~\ref{lem:sensitivity_bound} yields:
    \begin{equation}
        \log \frac{p(\calA(\calD) = \theta)}{p(\calA(\calD') = \theta)} \le \beta
        \|\theta_1 - \theta_2\|_F \le \frac{\beta 2 K}{N \lambda}.
    \end{equation}
    Thus if we choose $\beta = \frac{N \lambda \epsilon}{2 K}$, we achieve
    $\epsilon$-differential privacy.
\end{proof}

The model sensitivity method in Algorithm~\ref{alg:model_perturb} is
$(\epsilon, \delta)$-differentially private with $\delta = 0$. By letting
$\delta > 0$, we can use noise from a Gaussian distribution.

\begin{algorithm}
  \caption{Analytic Gaussian mechanism~\cite{balle2018improving}.}
    \label{alg:analytic_gaussian}
\begin{algorithmic}
  \STATE {\bf Inputs:} Privacy parameters $\epsilon$ and $\delta$.
  \STATE {\bf Output:} Scalar $\alpha$ used to compute the variance of the Guassian noise distribution.
  \STATE Let $\delta_0 = \Phi(0) - e^\epsilon \Phi(-\sqrt{2\epsilon})$.
  \IF {$\delta \geq \delta_0$}
    \STATE Define $B_\epsilon^+(v) = \Phi(\sqrt{\epsilon v}) - e^\epsilon \Phi(-\sqrt{\epsilon(v + 2)})$.
    \STATE Find $v^* = \sup\{v \in \mathbb{R}_{\geq 0}: B_\epsilon^+(v) \leq \delta \}$.
    \STATE Let $\alpha = \sqrt{1 + \nicefrac{v^*}{2}} - \sqrt{\nicefrac{v^*}{2}}$.
  \ELSE
    \STATE Define $B_\epsilon^-(u) = \Phi(-\sqrt{\epsilon u}) - e^\epsilon \Phi(-\sqrt{\epsilon(u + 2)})$.
    \STATE Find $u^* = \inf\{u \in \mathbb{R}_{\geq 0}: B_\epsilon^-(u) \leq \delta \}$.
  \ENDIF
\end{algorithmic}
\end{algorithm}

We use the analytic Gaussian mechanism of~\cite{balle2018improving}, which
results in better privacy parameters at the same noise scale than the
standard Gaussian mechanism~\cite{dwork2014algorithmic}. We rely on
Algorithm~\ref{alg:analytic_gaussian}, a subroutine of Algorithm 1
from~\cite{balle2018improving}, in order to compute the $\alpha$ required by
the analytic Gaussian mechanism. Following~\cite{balle2018improving}, we note
that $B_\epsilon^+(v)$ and $B_\epsilon^-(u)$ are monotonic functions and use
binary search to find the value of $v^*$ and $u^*$ in
Algorithm~\ref{alg:analytic_gaussian} up to arbitrary precision.

\begin{algorithm}
    \caption{Multi-class Gaussian model sensitivity.}
    \label{alg:gaussian_model_perturb}
\begin{algorithmic}
    \STATE {\bf Inputs:} Privacy parameters $\epsilon$ and $\delta$, regularization parameter $\lambda$, and dataset $\calD$.
    \STATE {\bf Output:} $\theta_{\textrm{priv}}$, the $(\epsilon, \delta)$-differentially private classifier.
    \STATE Compute $\alpha$ using Algorithm~\ref{alg:analytic_gaussian}.
    \STATE Sample $\bB$ from the distribution in Equation~\ref{eq:gaussian_noise} with $\sigma=\frac{2K \alpha}{N \lambda \sqrt{2\epsilon}}$.
    \STATE Compute $\theta_{\textrm{priv}} = \argmin_\theta J(\theta; \calD) + \bB$.
\end{algorithmic}
\end{algorithm}

\begin{appendix_theorem}
  \label{apxthm:gaussian_model_perturb}
  Given assumptions \ref{appxas:loss}, \ref{appxas:regularizer}, \ref{appxas:linear},
  \ref{appxas:lip}, and \ref{appxas:l2inp}, the Gaussian model sensitivity method in
  Algorithm~\ref{alg:gaussian_model_perturb} is $(\epsilon,
  \delta)$-differentially private for $\delta \in (0,1)$.
\end{appendix_theorem}

\begin{proof}
  From Lemma~\ref{lem:sensitivity_bound}, we have a bound on the sensitivity
  of $\theta(\calD) = \argmin_{\theta'} J(\theta'; \calD)$:
  \begin{equation}
    \label{eq:gaussian_model_sensitivity}
    \max_{\calD, \calD'}  \| \theta(\calD)  - \theta(\calD') \|_F \le \frac{2K}{N\lambda}.
  \end{equation}

  We apply Theorem 8 of~\cite{balle2018improving} which states that for any $\epsilon \geq 0$ and $\delta \in (0,1)$,
  the Gaussian mechanism with standard deviation $\sigma$ provides $(\epsilon, \delta)$-differential privacy
  if and only if:
  \begin{equation}
  \Phi\left( \frac{\Delta_2 f}{2\sigma} - \frac{\epsilon \sigma}{\Delta_2 f} \right) - e^{\epsilon} \Phi \left(-\frac{\Delta_2 f}{2\sigma} - \frac{\epsilon\sigma}{\Delta_2 f} \right) \leq \delta,
  \end{equation}
  where $\Delta_2 f$ is the $L_2$ sensitivity of the function $f(\calD)$, and $\Phi(\cdot)$ is the cumulative
  distribution function (CDF) of the standard univariate Gaussian distribution.

  Combining the upper bound on $\Delta_2 f$ from Equation~\ref{eq:gaussian_model_sensitivity} with Theorem 9
  of~\cite{balle2018improving} shows that the Gaussian model sensitivity
  method in Algorithm~\ref{alg:gaussian_model_perturb} is $(\epsilon, \delta)$-differentially private.
\end{proof}

\newcommand{\bJ}{\mathbf{J}}
\newcommand{\lmax}[1]{\lambda_{\max}(#1)}
\newcommand{\priv}[1]{#1_\textnormal{priv}}
\newcommand{\vect}[1]{\textnormal{vec}(#1)}

\subsection{Loss Perturbation}
We generalize the loss perturbation method of~\cite{chaudhuri2011} to the
multi-class setting. We use the slighly modified mobjective
of~\cite{kifer2012}:
\begin{equation}
    \label{eq:erm_loss_perturb}
    J(\theta; \calD) = \frac{1}{N} \sum_{n=1}^N \ell(\phi'(\bx_n; \theta), \by_n)  + \frac{\lambda}{N} R(\theta),
\end{equation}
where we assume that both $\ell(\cdot)$ and $R(\cdot)$ are convex with
continuous Hessians and that $\phi'(\cdot)$ is linear. The multi-class loss
perturbation method is given in Algorithm~\ref{alg:loss_perturb} where the
function $\text{tr}(\bA)$ is the trace of a square matrix $\bA$.

\begin{algorithm}
    \caption{Multi-class loss perturbation.}
    \label{alg:loss_perturb}
\begin{algorithmic}
    \STATE {\bf Inputs:} Privacy parameter $\epsilon$ and dataset $\calD$.
    \STATE {\bf Output:} $\theta_{\textrm{priv}}$, the $\epsilon$-differentially private classifier.
    \STATE Set $\rho \ge \frac{2LC}{\epsilon}$.
    \STATE Sample $\bB$ from the distribution in Equation~\ref{eq:sqrt_noise}
    with $\beta = \frac{\epsilon}{2K}$.
    \vspace{0.5mm}
    \STATE Compute $\theta_{\textrm{priv}} = \argmin_\theta J(\theta; \calD) + \frac{1}{N}\text{tr}(\bB^\top \theta) + \frac{\rho}{2N} \|\theta\|_F^2$.
\end{algorithmic}
\end{algorithm}

We require additional notation to state and prove the privacy theorem of
Algorithm~\ref{alg:loss_perturb}. First, we let $\vect{\bA}$ denote the
vectorization of the matrix $\bA$ constructed by stacking the rows of $\bA$
into a vector. If $\bA \in \mathbb{R}^{n \times m}$ then $\vect{\bA} \in
\mathbb{R}^{nm}$. Second, we rely on the standard Kronecker product which we
denote by the symbol $\otimes$. For two vectors we have $\bu \otimes \bv =
\vect{\bu \bv^\top}$. We let $\lmax{\bA}$ denote the maximum eigenvalue of a
symmetric matrix $\bA$.

\begin{appendix_theorem}
  \label{apxthm:loss_perturb}
  Given a convex loss $\ell(\cdot)$, a convex regularizer $R(\cdot)$ with
  continuous Hessians, assumptions \ref{appxas:linear}, \ref{appxas:lip},
  and \ref{appxas:l2inp}, and assuming that  $\lambda_{\max}(\nabla^2 \ell) \le L$, the
  loss perturbation method in Algorithm~\ref{alg:loss_perturb} is
  $\epsilon$-differentially private.
\end{appendix_theorem}

\begin{proof}
  The proof structure here closely follows that of Theorem 9
  in~\cite{chaudhuri2011} and Theorem 2 in~\cite{kifer2012} with the
  appropriate generalizations to account for the multi-class setting.

  Let $\bw = \vect{\theta}$ and $\bb = \vect{\bB}$. We aim to bound the ratio of
  the two densities:
  \begin{equation}
    \label{eq:loss_densities}
    \frac{g(\theta_{\text{priv}} \mid \calD)}{g(\theta_{\text{priv}} \mid \calD')} =
    \frac{g(\bw_{\text{priv}} \mid \calD)}{g(\bw_{\text{priv}} \mid \calD')}.
  \end{equation}
  By the same argument used in the proof of Theorem 9 of~\cite{chaudhuri2011},
  there is a bijection from $\bw_{\text{priv}}$ to $\bb$ for every $\calD$.
  Hence the ratio in Equation~\ref{eq:loss_densities} can be written as:
  \begin{equation}
    \label{eq:densities_expanded}
    \frac{g(\priv{\bw} \mid \calD)}{g(\priv{\bw} \mid \calD')}
    = \frac{\mu(\bb \mid \calD)}{\mu(\bb' \mid \calD')} \cdot
    \frac{|\det( \bJ(\priv{\bw} \rightarrow \bb \mid \calD))|^{-1}}{|\det( \bJ(\priv{\bw} \rightarrow \bb' \mid \calD'))|^{-1}},
  \end{equation}
  where $\mu(\cdot)$ is the conditional noise density given dataset $\calD$ and
  $\bJ(\priv{\bw}\rightarrow \bb \mid \calD)$ is the Jacobian of the mapping from
  $\priv{\bw}$ to $\bb$ for a given $\calD$.

  We bound each term in Equation~\ref{eq:densities_expanded}, starting with a
  bound on the ratio of the determinants of the Jacobians.

  For any $\priv{\bw}$ and $\calD$, we can solve for the noise term $\bb$ added
  to the objective in Algorithm~\ref{alg:loss_perturb}:
  \begin{equation}
  \label{eq:solved_noise}
  \bb  =  -\lambda \nabla R(\priv{\bw}) - \sum_{i=1}^N \nabla \ell \otimes \bx_i - \rho \priv{\bw},
  \end{equation}
  where we use the short-hand $\nabla_\ba \ell(\ba) = \nabla \ell$. By
  assumption~\ref{appxas:linear}, we have $\nabla_\theta \ell = \nabla \ell
  \bx^\top$ and hence $\nabla_\bw \ell = \nabla \ell \otimes \bx$.

  We can differentiate Equation~\ref{eq:solved_noise} to get:
  \begin{equation}
    \nabla \bb =  -\lambda \nabla^2 R(\priv{\bw}) - \sum_{i=1}^N \nabla^2 \ell \otimes (\bx_i \bx_i^\top) - \rho \bI_{DC},
  \end{equation}
  where $\bI_{DC}$ is the $DC$-dimensional identity matrix.

  Without loss of generality, let $\calD'$ contain one more example than
  $\calD$ which we denote by $\bx', \by'$, \emph{i.e.} $\calD' = \calD \cup
  \{(\bx', \by')\}$.
  We define the matrices $\bA$ and $\bE$ as:
  \begin{equation}
    \bA =  \lambda \nabla^2 R(\priv{\bw}) + \sum_{i=1}^{N} \nabla^2 \ell \otimes (\bx_i \bx_i^\top) + \rho \bI_{DC},
  \end{equation}
  and:
  \begin{equation}
    \bE = \nabla^2 \ell \otimes (\bx' \bx'^\top).
  \end{equation}
  Thus:
  \begin{equation}
    \label{eq:bound_with_a_e}
    \frac{|\det( \bJ(\priv{\bw} \rightarrow \bb' \mid \calD'))|}{|\det( \bJ(\priv{\bw} \rightarrow \bb \mid \calD))|}
    = \frac{|\det(\bA + \bE)|}{|\det(\bA)|}.
  \end{equation}
  Given that $\nabla^2 \ell \in \mathbb{R}^{C\times C}$ we
  know that $\text{Rank}(\nabla^2 \ell) \le C$. Since
  $\text{Rank}(\bx \bx^\top) = 1$, we have that $\text{Rank}(\nabla^2 \ell
  \otimes (\bx \bx^\top)) \le C$~\cite{horn1994topics}, and thus
  $\text{Rank}(\bE) \le C$. Thus we can apply
  Lemma~\ref{lem:rank_eig_bound} to bound the ratio:
  \begin{equation}
    \label{eq:a_e_bound}
    \frac{|\det(\bA + \bE)|}{|\det(\bA)|} \le (1 + \lmax{\bA^{-1}\bE})^{C}.
  \end{equation}
  Given that $R(\cdot)$ and $\ell(\cdot)$ are convex, the eigenvalues of $\bA$
  are at least $\rho$, and thus $\lmax{\bA^{-1}\bE} \le
  \frac{1}{\rho}\lmax{\bE}$.
  Given the assumption that $\lambda_{\max}(\nabla^2 \ell) \le L$ and by
  assumption~\ref{appxas:l2inp} $\lambda_{\max}(\bx \bx^\top) \le 1$, we have that
  $\lambda_{\max}(\nabla^2 \ell \otimes (\bx \bx^\top)) \le L$~\cite{horn1994topics},
  thus $\lmax{\bE} \le L$.
  Combining these observations with Equations~\ref{eq:bound_with_a_e} and
  \ref{eq:a_e_bound}, we have a bound on the ratio:
  \begin{equation}
    \frac{|\det( \bJ(\priv{\bw} \rightarrow \bb' \mid \calD'))|}{|\det( \bJ(\priv{\bw} \rightarrow \bb \mid \calD))|}
     \le \left(1 + \frac{L}{\rho}\right)^{C}.
  \end{equation}

  Since $\rho \ge \frac{2LC}{\epsilon}$, we have:
  \begin{equation}
    \left(1 + \frac{L}{\rho}\right)^{C} \le \left(1 + \frac{\epsilon}{2C}\right)^{C} \le e^{\epsilon / 2},
  \end{equation}
  using the identity $1+x \le e^x$ for all $x$.
  Thus we have a final bound on the ratio:
  \begin{equation}
    \label{eq:jacobian_ratio_bound}
    \frac{|\det( \bJ(\priv{\bw} \rightarrow \bb' \mid \calD'))|}{|\det( \bJ(\priv{\bw} \rightarrow \bb \mid \calD))|}
     \le e^{\epsilon/2}.
  \end{equation}

  Next, we bound the ratio $\frac{\mu(\bb \mid \calD)}{\mu(\bb' \mid \calD')}$.
  From Equation~\ref{eq:solved_noise} we have:
  \begin{equation}
    \bb' - \bb = \nabla \ell \otimes \bx'.
  \end{equation}
  Thus:
  \begin{equation}
    \label{eq:noise_norm_bound}
    \|\bb'\|_2 - \|\bb\|_2 \le \|\bb' - \bb\|_2  = \left\| \nabla \ell \otimes \bx' \right\|_2 \le \| \nabla \ell \|_2 \| \bx' \|_2 \le K,
  \end{equation}
  where we use the triangle inequality in the first step, the fact
  that $\|\bu \otimes \bv\|_2 = \|\bu \bv^\top\|_F = \|\bu\|_2 \|\bv\|_2$ in
  the third step, and assumptions~\ref{appxas:lip} and \ref{appxas:l2inp} in the fifth
  step.
  Following the argument in the proof of Theorem 9 of~\cite{chaudhuri2011}, we have:
  \begin{equation}
    \label{eq:noise_density_chaudhuri}
    \frac{\mu(\bb \mid \calD)}{\mu(\bb' \mid \calD')} \le e^{\frac{\epsilon}{2K}(\|\bb\|_2 -\|\bb'\|_2)}.
  \end{equation}
  Combining Equations~\ref{eq:noise_norm_bound} and \ref{eq:noise_density_chaudhuri}, we have:
  \begin{equation}
    \label{eq:noise_density_bound}
    \frac{\mu(\bb \mid \calD)}{\mu(\bb' \mid \calD')} \le e^{\epsilon / 2}.
  \end{equation}
  Finally, combining Equations~\ref{eq:densities_expanded}, \ref{eq:jacobian_ratio_bound} and \ref{eq:noise_density_bound}, yields:
  \begin{equation}
    \frac{g(\priv{\bw} \mid \calD)}{g(\priv{\bw} \mid \calD')}
    = \frac{\mu(\bb \mid \calD)}{\mu(\bb' \mid \calD')} \cdot
    \frac{|\det( \bJ(\priv{\bw} \rightarrow \bb \mid \calD))|^{-1}}{|\det( \bJ(\priv{\bw} \rightarrow \bb' \mid \calD'))|^{-1}} \le e^{\epsilon / 2} e^{\epsilon / 2} \le e^\epsilon,
  \end{equation}
  completing the proof.
\end{proof}

We can use noise from a Gaussian distribution by letting $\delta >
0$~\cite{kifer2012}. We generalize the approach of~\cite{kifer2012} to the
multi-class setting.

\begin{algorithm}
    \caption{Multi-class Gaussian loss perturbation.}
    \label{alg:gaussian_loss_perturb}
\begin{algorithmic}
    \STATE {\bf Inputs:} Privacy parameters $\epsilon$ and $\delta$, and dataset $\calD$.
    \STATE {\bf Output:} $\theta_{\textrm{priv}}$, the $(\epsilon, \delta)$-differentially private classifier.
    \STATE Set $\rho \ge \frac{2LC}{\epsilon}$.
    \STATE Sample $\bB$ from the distribution in Equation~\ref{eq:gaussian_noise}
      with $\sigma = \frac{K}{\epsilon} \sqrt{8 \log (2 / \delta) + 4 \epsilon}$.
    \vspace{0.5mm}
    \STATE Compute $\theta_{\textrm{priv}} = \argmin_\theta J(\theta, \calD) + \frac{1}{N}\text{tr}(\bB^\top \theta) + \frac{\rho}{2N} \|\theta\|_F^2$.
\end{algorithmic}
\end{algorithm}

\begin{appendix_theorem}
  Given a convex loss $\ell(\cdot)$, a convex regularizer $R(\cdot)$ with
  continuous Hessians, assumptions \ref{appxas:linear}, \ref{appxas:lip},
  and \ref{appxas:l2inp}, and assuming that  $\lambda_{\max}(\nabla^2 \ell) \le L$, the
  loss perturbation method in Algorithm~\ref{alg:gaussian_loss_perturb} is
  $(\epsilon, \delta)$-differentially private.
\end{appendix_theorem}

\begin{proof}
  Let $\bw = \vect{\theta}$ and $\bb = \vect{\bB}$. We aim to bound the ratio of
  the two densities:
  \begin{equation}
    \label{eq:g_loss_densities}
    \frac{g(\theta_{\text{priv}} \mid \calD)}{g(\theta_{\text{priv}} \mid \calD')} =
    \frac{g(\bw_{\text{priv}} \mid \calD)}{g(\bw_{\text{priv}} \mid \calD')},
  \end{equation}
  with probability at least $1-\delta$.

  As in Theorem~\ref{apxthm:loss_perturb}, there is a bijection from
  $\bw_{\text{priv}}$ to $\bb$ for every $\calD$.
  Hence the ratio in Equation~\ref{eq:g_loss_densities} can be written as:
  \begin{equation}
    \label{eq:g_densities_expanded}
    \frac{g(\priv{\bw} \mid \calD)}{g(\priv{\bw} \mid \calD')}
    = \frac{\mu(\bb \mid \calD)}{\mu(\bb' \mid \calD')} \cdot
    \frac{|\det( \bJ(\priv{\bw} \rightarrow \bb \mid \calD))|^{-1}}{|\det( \bJ(\priv{\bw} \rightarrow \bb' \mid \calD'))|^{-1}},
  \end{equation}
  where $\mu(\cdot)$ is the conditional noise density given dataset $\calD$ and
  $\bJ(\priv{\bw}\rightarrow \bb \mid \calD)$ is the Jacobian of the mapping from
  $\priv{\bw}$ to $\bb$ for a given $\calD$.

  From Equation~\ref{eq:jacobian_ratio_bound} in the proof of
  Theorem~\ref{apxthm:loss_perturb}, we have that:
  \begin{equation}
    \label{eq:g_jacobian_ratio_bound}
    \frac{|\det( \bJ(\priv{\bw} \rightarrow \bb \mid \calD))|^{-1}}{|\det( \bJ(\priv{\bw} \rightarrow \bb' \mid \calD'))|^{-1}}
     \le e^{\epsilon/2}.
  \end{equation}

  Following the same arguments as the proof of Theorem 2
  in~\cite{kifer2012}, we have with at least probability $1-\delta$:
  \begin{equation}
    \frac{\mu(\bb \mid \calD)}{\mu(\bb' \mid \calD')} \le e^{\frac{1}{2\beta^2}\left(\beta K \sqrt{8\log(2 / \delta)} + K^2\right)}.
  \end{equation}
  If we choose $\beta \ge \frac{K}{\epsilon} \sqrt{8 \log (2/\delta) + 4\epsilon}$ then:
  \begin{equation}
    \label{eq:g_density_bound}
    \frac{\mu(\bb \mid \calD)}{\mu(\bb' \mid \calD')} \le e^{\epsilon / 2}.
  \end{equation}
  Combining Equations~\ref{eq:g_densities_expanded}, \ref{eq:g_jacobian_ratio_bound} and \ref{eq:g_density_bound} yields:
  \begin{equation}
    \frac{g(\priv{\bw} \mid \calD)}{g(\priv{\bw} \mid \calD')} \le e^\epsilon,
  \end{equation}
  with probability at least $1-\delta$, completing the proof.

\end{proof}

\begin{lemma}
\label{lem:rank_eig_bound}
Given symmetric and positive semidefinite matrices, $\bA$ and $\bE$, if $\bA$ is full rank, and $\bE$ has rank at most $C$, then:
\begin{equation}
  \frac{\det(\bA + \bE)}{\det{\bA}} \le (1 + \lmax{\bA^{-1}\bE})^C,
\end{equation}
  where $\lmax{\bA^{-1}\bE}$ is the maximum eigenvalue of $\bA^{-1}\bE$.
\end{lemma}

\begin{proof}
  \begin{align*}
  \frac{\det(\bA + \bE)}{\det{\bA}} &= \det{(\bI + \bA^{-1}\bE)} \\
    &= \prod_{c=1}^C (1 + \lambda_c(\bA^{-1}\bE)) \\
    &\le (1 + \lmax{\bA^{-1}\bE})^C.
  \end{align*}
  In the second step we use the facts that $\lambda_i(\bI + \bA^{-1} \bE) = 1 +
  \lambda_i(\bA^{-1}\bE)$ and that $\bA^{-1} \bE$ has rank at most $C$ and
  hence has at most $C$ nonzero eigenvalues.
\end{proof}

\section{Private Prediction}

\subsection{Prediction Sensitivity}

The prediction sensitivity method protects the privacy of the underlying
training dataset, $\calD$, by adding a unique noise vector to the logit predictions, $\phi'(\hat{\bx}; \theta)$,
for each test example $\hat{\bx}$.

The prediction sensitivity method is given in
Algorithm~\ref{alg:prediction_perturb}. The algorithm accepts as input a budget
$B$ of test examples and guarantees $\epsilon$-differential privacy of $\calD$
until the budget is exhausted after which the privacy of $\calD$ degrades.

\begin{algorithm}
    \caption{Multi-class prediction sensitivity.}
    \label{alg:prediction_perturb}
\begin{algorithmic}
  \STATE {\bf Inputs:} Privacy parameter $\epsilon$, non-private model $\phi'(\cdot)$, test example $\hat{\bx}$, and budget $B$.
  \STATE {\bf Output:} $\hat{\by}_\textrm{priv}$, the prediction which is
    $\epsilon$-differentially private with respect to the training dataset
    $\calD$.
  \STATE Sample $\bb$ from the distribution in Equation~\ref{eq:sqrt_noise} with $\beta = \frac{N\lambda\epsilon}{2KB}$.
    \STATE Compute $\hat{\by}_\textrm{priv} = \phi'(\hat{\bx}; \theta) + \bb$.
\end{algorithmic}
\end{algorithm}

\begin{appendix_theorem}
  \label{apxthm:prediction_perturb}
  Given assumptions \ref{appxas:loss}, \ref{appxas:regularizer}, \ref{appxas:linear},
  \ref{appxas:lip}, and \ref{appxas:l2inp}, the prediction sensitivity method in
  Algorithm~\ref{alg:prediction_perturb} is $\epsilon$-differentially private.
\end{appendix_theorem}

\begin{proof}
  The proof proceeds the same way as that of Theorem~\ref{apxthm:model_perturb}.
  First, we bound the privacy loss in terms of the sensitivity of the minimizer
  of $J(\cdot)$ in the Frobenius norm and then we apply
  Lemma~\ref{lem:sensitivity_bound}.

  Let $\theta_1\!=\!\argmin_\theta J(\theta; \calD)$ and $\theta_2\!=\!\argmin_\theta J(\theta; \calD')$.
  For all datasets $\calD$ and $\calD'$ which differ by one example and all
  $\bx$, we have:
  \begin{align*}
      &\log \frac{p(\calA(\calD, \bx) = \by)}{p(\calA(\calD', \bx) = \by)} \\
      &=\log \frac{e^{-\beta \|\theta_1^\top \bx + \bb \|_2}}{e^{-\beta \|\theta_2^\top \bx + \bb\|_2}} \\
      &= \beta \left(\|\theta^\top_2 \bx + \bb\|_2 - \|\theta_1^\top \bx + \bb \|_2 \right)  \\
      &\le \beta \|\theta_1^\top \bx - \theta_2^\top \bx \|_2 \\
      &\le \beta \|(\theta_1 - \theta_2)^\top \bx \|_2,
  \end{align*}
  where we use the triangle inequality in the second to last step.  We also have that:
  \begin{equation}
      \|(\theta_1 - \theta_2)^\top \bx \|_2 \le \|\theta_1 - \theta_2 \|_F \| \bx\|_2 \le \|\theta_1 - \theta_2 \|_F,
  \end{equation}
  where we use Cauchy-Schwarz in the first inequality and assumption~\ref{appxas:l2inp} in the second.
  Thus:
  \begin{equation}
      \log \frac{p(\calA(\calD, \bx) = \by)}{p(\calA(\calD', \bx) = \by)} \le
      \beta \|\theta_1 - \theta_2 \|_F. \label{eq:prediction_privloss}
  \end{equation}
  Combining Equation~\ref{eq:prediction_privloss} with Lemma~\ref{lem:sensitivity_bound} yields:
  \begin{equation}
      \log \frac{p(\calA(\calD, \bx) = \by)}{p(\calA(\calD', \bx) = \by)} \le \beta
      \|\theta_1 - \theta_2\|_F \le \frac{\beta 2 K}{N \lambda}.
  \end{equation}
  If we choose $\beta = \frac{N \lambda \epsilon}{2 K}$, we achieve
  $\epsilon$-differential privacy when releasing the predictions on a single
  example $\bx$. If we choose $\beta = \frac{N \lambda \epsilon}{2 K B}$ then
  by standard compositional arguments of differential
  privacy~\cite{dwork2006} we achieve $\epsilon$-differential
  privacy with a budget of $B$ test queries.
\end{proof}

\begin{algorithm}
    \caption{Multi-class Gaussian prediction sensitivity.}
    \label{alg:gaussian_prediction_perturb}
\begin{algorithmic}
  \STATE {\bf Inputs:} Privacy parameters $\epsilon$ and $\delta$,
  non-private model $\phi'(\cdot)$, test example $\hat{\bx}$, and budget $B$.
  \STATE {\bf Output:} $\hat{\by}_\textrm{priv}$, the prediction which is
    $(\epsilon, \delta)$-differentially private with respect to the training dataset
    $\calD$.
  \STATE Let $\sigma' = \frac{2K \alpha}{N \lambda \sqrt{2\epsilon^*}}$ with $\alpha$ computed using Algorithm~\ref{alg:analytic_gaussian}, $\epsilon^* = \epsilon / B$, and $\delta^* = \delta / B$.
  \STATE {\bf Minimize} $\sigma''=\frac{2K\alpha}{N \lambda \sqrt{2\epsilon^*}}$ over $\delta' \in (0, \delta - B\delta^*)$ using linear search {\bf where}:
  \STATE\hspace{\algorithmicindent} $\epsilon^* = \sqrt{2 / B} \left(\sqrt{\ln(1/\delta') + \epsilon} - \sqrt{\ln(1/\delta')}\right)$,
  \STATE\hspace{\algorithmicindent} $\delta^* = (\delta - \delta') / B$,
  \STATE\hspace{\algorithmicindent} and where $\alpha$ is computed using Algorithm~\ref{alg:analytic_gaussian}.
  \STATE Sample $\bb$ from the distribution in Equation~\ref{eq:gaussian_noise} with $\sigma = \min(\sigma', \sigma'')$.
  \STATE Compute $\hat{\by}_\textrm{priv} = \phi'(\hat{\bx}; \theta) + \bb$.
\end{algorithmic}
\end{algorithm}

\begin{appendix_theorem}
  \label{apxthm:gaussian_prediction_perturb}
  Given assumptions \ref{appxas:loss}, \ref{appxas:regularizer}, \ref{appxas:linear},
  \ref{appxas:lip}, and \ref{appxas:l2inp}, the Gaussian prediction sensitivity method in
  Algorithm~\ref{alg:gaussian_prediction_perturb} is $(\epsilon, \delta)$-differentially
  private for $\delta \in (0, 1)$.
\end{appendix_theorem}

\begin{proof}
  As in the proof of Theorem~\ref{apxthm:gaussian_model_perturb}, we bound the $L_2$
  sensitivity and apply Theorem 9 of~\cite{balle2018improving}.

  Let $\theta_1 = \argmin_\theta J(\theta; \calD)$ and similarly $\theta_2 = \argmin_\theta J(\theta; \calD')$.
  From Theorem~\ref{apxthm:prediction_perturb} we have:
  \begin{equation}
    \label{eq:gaussian_pred_sensitivity}
    \max_{\calD, \calD', \bx}  \| \theta_1^\top \bx - \theta_2^\top \bx \|_2 \le \frac{2K}{N\lambda}.
  \end{equation}

  Recall from our description of the Gaussian model sensitivity method that Theorem 8 of~\cite{balle2018improving} states that the Gaussian mechanism with standard deviation $\sigma'$ is $(\epsilon^*, \delta^*)$-differentially
  private if and only if:
  \begin{equation}
    \Phi\left( \frac{\Delta_2 f}{2\sigma'} - \frac{\epsilon^* \sigma'}{\Delta_2 f} \right) - e^{\epsilon^*} \Phi \left(-\frac{\Delta_2 f}{2\sigma'} - \frac{\epsilon^*\sigma'}{\Delta_2 f} \right) \leq \delta^*,
  \end{equation}
  where $\Delta_2 f$ is the $L_2$ sensitivity of $f(\cdot)$.

  Combining Equation~\ref{eq:gaussian_pred_sensitivity} with Theorem 9 of~\cite{balle2018improving}, we can obtain $(\epsilon^*, \delta^*)$-differential privacy on a single prediction by setting:
  \begin{equation}
  \label{eq:individual_sigma}
  \sigma'' = \frac{2K \alpha}{N \lambda \sqrt{2\epsilon^*}},
  \end{equation}
  where $\alpha$ is defined in Algorithm~\ref{alg:analytic_gaussian}.

  Theorem 1.1 of~\cite{dwork2016concentrated} states that $B$ compositions of a
  $(\epsilon^*, \delta^*)$-differentially private mechanism satisfies $(\epsilon,
  \delta)$-differential privacy where:
  \begin{equation}
    \epsilon =  \sqrt{2B \ln(1/\delta')} \epsilon^* + B \epsilon^* (e^{\epsilon^*}-1) / 2,
  \end{equation}
  and:
  \begin{equation}
    \delta = \delta' + B \delta^*.
  \end{equation}

  We can solve for $\epsilon^*$ and $\delta^*$ in terms of $\epsilon$, $\delta$
  and $\delta'$. Given that $\epsilon$ and $\delta$ are pre-specified, this leaves
  $\delta'$ to be determined.

  Using the fact that $1+x \le e^x$ for all $x$, we have:
  \begin{equation}
    \sqrt{2B \ln(1/\delta')} \epsilon^* + B (\epsilon^*)^2  / 2 \le
    \sqrt{2B \ln(1/\delta')} \epsilon^* + B \epsilon^* (e^{\epsilon^*}-1) / 2.
  \end{equation}
  This is a quadratic in $\epsilon^*$ which we can solve to obtain:
  \begin{equation}
    \epsilon^* = \sqrt{2 / B} \left(\sqrt{\ln(1/\delta') + \epsilon} - \sqrt{\ln(1/\delta')}\right).
  \end{equation}
  We also have:
  \begin{equation}
    \delta^* = (\delta - \delta') / B.
  \end{equation}
  The algorithm will be $(\epsilon, \delta)$-differentially private with
  respect to $B$ predictions for any $\epsilon^*$ and $\delta^*$ which satisfy
  the above equations. Thus, for a given $\epsilon$, $\delta$ and $B$, we can
  choose $\delta'$ which minimizes $\sigma''$ in
  Equation~\ref{eq:individual_sigma} and achieve $(\epsilon,
  \delta)$-differential privacy with a budget $B$.

  For small values of the budget $B$, the standard composition theorem of $(\epsilon, \delta)$-differential
  privacy (\emph{e.g.}, Theorem 3.16 of~\cite{dwork2011}) may actually lead to smaller standard deviations in the Gaussian noise distribution.
  Combining the standard composition theorem with Equation~\ref{eq:gaussian_pred_sensitivity} and Theorem 9 of~\cite{balle2018improving},
  we can obtain $(\epsilon, \delta)$-differential privacy on a single prediction by setting:
  \begin{equation}
  \sigma' = \frac{2K \alpha}{N \lambda \sqrt{2\epsilon^*}},
  \end{equation}
  where $\alpha$ is obtained via Algorithm~\ref{alg:analytic_gaussian} with
  $\epsilon^* = \epsilon / B$ and $\delta^* \ \delta / B$.

  Because both $\sigma'$ and $\sigma''$ provide the required differential privacy guarantee, so does the mechanism
  in Algorithm~\ref{alg:gaussian_prediction_perturb} that sets $\sigma = \min(\sigma', \sigma'')$.
\end{proof}

\subsection{Subsample-and-Aggregate}

The subsample-and-aggregate approach of~\cite{dwork2018} splits the training
dataset $\calD$ into $T$ disjoint subsets of size $\nicefrac{|\calD|}{T}$.  A
set of $T$ models, $\{\phi'_1(\cdot; \theta_1), \ldots, \phi'_T(\cdot;
\theta_T)\}$, are learned, one for each subset, where $\phi'_t(\cdot)$ outputs
a one-hot vector of size $C$. To make a prediction, the classifiers are
combined using a soft majority vote. For a given $\bx$, the class label $\by$
is predicted with probability proportional to:
\begin{equation}
    \label{eq:subsample_predict}
    \exp\left(\beta \cdot | \{t : t \in \{1, \dots, T \}, \phi'_t(\bx; \theta_t) = \by \}|\right).
\end{equation}
Privacy by classifying with probability proportional to an exponentiated utility function
is known as the exponential mechanism~\cite{dwork2014algorithmic} and is due
to~\cite{mcsherry2007mechanism}.

\begin{algorithm}
    \caption{Subsample-and-aggregate.}
    \label{alg:subsample_and_aggregate}
\begin{algorithmic}
  \STATE {\bf Inputs:} Privacy parameters $\epsilon$ and $\delta$, training dataset $\calD$, a number of models $T$, and a set of queries $\mathcal{Q} =\{\hat{\bx}_1, \dots, \hat{\bx}_B\}$.
  \STATE {\bf Output:} $\{\hat{\by}_1, \dots \hat{\by}_B\}$, the predictions which are
    $\epsilon$-differentially private with respect to the training dataset $\calD$.
  \STATE Partition $\calD$ into $T$ disjoint subsets $\{\calD_t : t=1,\ldots,T \}$ of size $\lfloor \nicefrac{|\calD|}{T} \rfloor$.
  \FORALL {$\calD_t$}
    \STATE Train classifier $\phi'_t(\cdot; \theta_t)$ on $\calD_t$.
  \ENDFOR
  \IF {$\delta > 0$}
    \STATE Set $\beta = \max\left(\epsilon / B, \sqrt{2 / B} \left(\sqrt{\ln(1/\delta) + \epsilon} - \sqrt{\ln(1/\delta)}\right)\right)$.
  \ELSE
    \STATE Set $\beta = \epsilon / B$.
  \ENDIF
  \FOR {b = 1, \dots, B}
    \STATE Sample $\hat{\by}_b$ with probability proportional to Equation~\ref{eq:subsample_predict}.
  \ENDFOR
\end{algorithmic}
\end{algorithm}

\begin{appendix_theorem}
  \label{appxthm:subandagg}
  The subsample-and-aggregate method in Algorithm~\ref{alg:subsample_and_aggregate} with $\beta = \epsilon / B$ is $\epsilon$-differentially
  private.
\end{appendix_theorem}

\begin{proof}
  Consider any two datasets $\calD$ and $\calD'$ which differ by at most one
  example. Since the classifiers $\phi'_t(\cdot)$ are trained on disjoint
  subsets of $\calD$, at most one classifier can change its prediction for any
  given instance $\bx$ when we switch from training on $\calD$ to $\calD'$.
  If we let $\phi'_t(\cdot; \theta)$ denote classifiers trained on $\calD$ and
  $\phi'_t(\cdot; \theta')$ denote classifiers trained on $\calD'$, we have:
  \begin{equation}
    \label{eq:subsample_sensitivity}
    \argmax_{\calD, \calD', \bx, \by} \Big| | \{t : t \in \{1, \dots, T \}, \phi'_t(\bx; \theta_t) = \by \}| - | \{t : t \in \{1, \dots, T \}, \phi'_t(\bx; \theta'_t) = \by \}|  \Big| \le 1.
  \end{equation}
  In other words, the sensitivity of the majority vote is $1$.
  For all adjacent datasets, $\calD$ and $\calD'$, all examples $\bx$, and all class labels $\by$:
  \begin{align}
      \log &\frac{p(\calA(\bx, \calD) = \by)}{p(\calA(\bx, \calD') = \by)} =
      \log \frac{\exp(\beta \cdot | \{t : t \in \{1, \dots, T \}, \phi'_t(\bx; \theta_t) = \by \}|)}
        {\exp(\beta \cdot | \{t : t \in \{1, \dots, T \}, \phi'_t(\bx; \theta'_t) = \by \}|)} \\
      &\le \beta \cdot | \{t : t \in \{1, \dots, T \}, \phi'_t(\bx; \theta_t) = \by \}|
        - \beta \cdot | \{t : t \in \{1, \dots, T \}, \phi'_t(\bx; \theta'_t) = \by \}| \\
      &\le \beta \Big| | \{t : t \in \{1, \dots, T \}, \phi'_t(\bx; \theta_t) = \by \}|
        - | \{t : t \in \{1, \dots, T \}, \phi'_t(\bx; \theta'_t) = \by \}|  \Big| \\
      &\le \beta,
  \end{align}
  where we use Equation~\ref{eq:subsample_sensitivity} in the second to last step.

  Thus if we set $\beta = \epsilon / B$, by standard compositional arguments~\cite{dwork2006},
  we achieve $\epsilon$-differential privacy with a budget of $B$ test queries
  by predicting $\by$ with probability proportional to Equation~\ref{eq:subsample_predict}.
\end{proof}

We can achieve a better scaling of $\beta$ with the budget $B$ by letting
$\delta >0$ and relying on the advanced composition
theorem~\cite{dwork2010boosting}.

\begin{appendix_theorem}
  The subsample-and-aggregate method in Algorithm~\ref{alg:subsample_and_aggregate} with $\beta=\max(\beta', \beta'')$, where $\beta' = \epsilon / B$ and $\beta'' =  \sqrt{2 / B}\left(\sqrt{\ln(1/\delta) + \epsilon} - \sqrt{\ln(1/\delta)}\right)$, is $(\epsilon, \delta)$-differentially private.
\end{appendix_theorem}

\begin{proof}
  From Theorem~\ref{appxthm:subandagg}, the subsample-and-aggregate method with $B=1$ is $(\beta, 0)$-differentially private.

  Theorem 1.1 of~\cite{dwork2016concentrated} states that $B$ compositions of a $(\beta, 0)$-differentially private mechanism satisfies $\left(\sqrt{2B \ln(1/\delta)} \beta + B \beta (e^{\beta}-1) / 2, \delta\right)$.
  We can solve for $\beta$ such that the subsample-and-aggregate algorithm with a budget of $B$ satisfies $(\epsilon, \delta)$-differential privacy.
  Using the fact that $1+x \le e^x$ for all $x$, we have:
  \begin{equation}
    \sqrt{2B \ln(1/\delta)} \beta + B \beta^2  / 2 \le
    \sqrt{2B \ln(1/\delta)} \beta + B \beta (e^{\beta}-1) / 2.
  \end{equation}
  Hence if $\sqrt{2B \ln(1/\delta)} \beta + B \beta^2  / 2 \le \epsilon$ then the resulting algorithm will satisfy $(\epsilon, \delta)$-differential privacy.
  This is a quadratic in $\beta$ which we can solve to obtain:
  \begin{equation}
    \beta = \sqrt{2 / B} \left(\sqrt{\ln(1/\delta) + \epsilon} - \sqrt{\ln(1/\delta)}\right).
  \end{equation}
  If this value of $\beta$ is smaller than $\epsilon / B$, we can use Theorem 7 instead and set $\beta=\epsilon / B$.
\end{proof}

\section{Multi-class Logistic Loss}

In practice, we need to specify $\ell(\theta^\top \bx, \by)$ and bound $\| \nabla \ell
\|_2$ for the model sensitivity and loss perturbation methods.
The commonly used multi-class logistic loss is given by:
\begin{equation}
    \label{eq:crossent}
    \ell(\ba, \by) = \sum_{i=1}^C y_i \log \frac{e^{a_i}}{\sum_{j=1}^C e^{a_j}},
\end{equation}
where, for linear models, $\ba = \theta^\top \bx$.

\begin{appendix_theorem}
  The Lipschitz constant of the multi-class logistic loss (Equation~\ref{eq:crossent}) is $K = \sqrt{2}$.
\end{appendix_theorem}

\begin{proof}
  Since $\by \in \Delta^C$, we can write
  Equation~\ref{eq:crossent} as:
  \begin{equation}
      \ell(\ba, \by) = \by^\top \ba  - \log Z,
  \end{equation}
  where $Z = \sum_{j=1}^C e^{a_j}$, which has a gradient given by:
  \begin{equation}
    \nabla \ell = \by - \frac{1}{Z} e^{\ba}.
  \end{equation}
  By Corollary~\ref{cor:simplex_bound} the maximum $L_2$ distance between any
  two points on the probability simplex $\sqrt{2}$.
  Given that both $\by, \frac{1}{Z} e^{\ba} \in \Delta^C$:
  \begin{equation}
    \| \nabla \ell \|_2 = \left\| \by - \frac{e^{\ba}}{\sum_{j=1}^C e^{a_j}} \right\|_2 \le \sqrt{2}.
  \end{equation}
\end{proof}

For the loss perturbation method we also need to bound the eigenvalues and the rank
of the Hessian of Equation~\ref{eq:crossent} with respect to $\ba$.

The Hessian of the multi-class logistic loss with respect to $\ba$ is given by:
\begin{equation}
\label{eq:crossent_hessian}
\nabla^2 \ell = \text{diag}(\bp) - \bp \bp^\top,
\end{equation}
where $\bp = \frac{1}{Z} e^{\ba}$~\cite{boyd2004convex}. Since $\nabla^2 \ell
\in \mathbb{R}^{C\times C}$, the rank is at most $C$.

\begin{appendix_theorem}
  The eigenvalues of the Hessian of the multi-class logistic loss are
  bounded by $0.5$, \emph{i.e.}:
  \begin{equation}
    \lambda_{\max}(\nabla^2 \ell(\theta, \bx, \by)) \le 0.5,
  \end{equation}
  for all $\theta, \bx$ and $\by$.
\end{appendix_theorem}

\begin{proof}
  The Hessian of the multi-class logistic loss is given by
  Equation~\ref{eq:crossent_hessian}. We use the fact that the eigenvalues of a
  square matrix $\bA$ are contained in the union of the Gerschgorin discs
  constructed from the rows of $\bA$~\cite{horn2012matrix}.
  The Gerschgorin disc of the $i$-th row of $\bA$ has a center at $A_{ii}$ and a radius of $\sum_{j\ne i} |A_{ij}|$.
  Hence, an upper bound on the $i$-th Gerschgorin disc of $\text{diag}(\bp) - \bp \bp^\top$ is given by:
  \begin{equation}
    (p_i - p_i^2) + p_i \sum_{\substack{j=1 \\ j \ne i}}^C p_j
    = p_i (1 - p_i) + p_i (1-p_i)
    \le 0.25 + 0.25
    = 0.5,
  \end{equation}
  where we use the facts that $\sum_{\substack{j=1 \\ j \ne i}}^C p_j =
  1-p_i$ and $(1-p_i)p_i \le 0.25$ for $p_i \in [0, 1]$. Hence:
  \begin{equation}
    \lambda_{\max}(\nabla^2 \ell) \le 0.5.
  \end{equation}
\end{proof}

We rely on the following Corollary to bound the $L_2$ distance between two points on
the probability simplex, $\Delta^C$.
\begin{corollary}
  \label{cor:simplex_bound}
  For any two points $\bu$ and $\bv$ on the $(C\!-\!1)$-dimensional probability
  simplex their $L_2$ distance is no more than $\sqrt{2}$, \emph{i.e.}:
  \begin{equation}
    \max_{\bu, \bv \in \Delta^C} \|\bu -\bv\|_2 \le \sqrt{2}.
  \end{equation}
\end{corollary}

\begin{proof}
  Bauer's maximum principle states that a convex function on a convex set
  attains its maximum at an extreme point~\cite{kruvzik2000bauer}.  The $L_2$
  norm is a convex function, the simplex $\Delta^C$ is a convex set. The
  extreme points of the simplex are when $\bu$ and $\bv$ are vertices which for
  $\Delta^C$ implies they are standard basis vectors. The $L_2$ distance
  between two standard basis vectors is $\sqrt{2}$.
\end{proof}


\begin{thebibliography}{10}

\bibitem{abadi2016}
M.~Abadi, A.~Chu, I.~Goodfellow, H.~McMahan, I.~Mironov, K.~Talwar, and
  L.~Zhang.
\newblock Deep learning with differential privacy.
\newblock In {\em Proceedings of the CCS}, pages 308--318, 2016.

\bibitem{balle2018improving}
B.~Balle and Y.-X. Wang.
\newblock Improving the gaussian mechanism for differential privacy.
\newblock In {\em Proceedings of International Conference on Machine Learning},
  2018.

\bibitem{bassily2014}
R.~Bassily, A.~Smith, and A.~Thakurta.
\newblock Differentially private empirical risk minimization: Efficient
  algorithms and tight error bounds.
\newblock Technical report, 2014.

\bibitem{bassily2018}
R.~Bassily, O.~Thakkar, and A.~G. Thakurta.
\newblock Model-agnostic private learning.
\newblock In {\em Advances in Neural Information Processing Systems (NeurIPS)},
  volume~31, 2018.

\bibitem{boyd2004convex}
S.~Boyd, S.~P. Boyd, and L.~Vandenberghe.
\newblock {\em Convex optimization}.
\newblock Cambridge university press, 2004.

\bibitem{carlini2020}
N.~Carlini, M.~Jagielski, and I.~Mironov.
\newblock Cryptanalytic extraction of neural network models.
\newblock In {\em arXiv:2003.04884}, 2020.

\bibitem{carlini2019secret}
N.~Carlini, C.~Liu, {\'U}.~Erlingsson, J.~Kos, and D.~Song.
\newblock The secret sharer: Evaluating and testing unintended memorization in
  neural networks.
\newblock In {\em 28th $\{$USENIX$\}$ Security Symposium ($\{$USENIX$\}$
  Security 19)}, pages 267--284, 2019.

\bibitem{chaudhuri2011}
K.~Chaudhuri, C.~Monteleoni, and A.~D. Sarwate.
\newblock Differentially private empirical risk minimization.
\newblock {\em Journal of Machine Learning Research}, 12:1069--1109, 2011.

\bibitem{dagan2019}
Y.~Dagan and V.~Feldman.
\newblock {PAC} learning with stable and private predictions.
\newblock In {\em arXiv 1911.10541}, 2019.

\bibitem{dwork2011}
C.~Dwork.
\newblock Differential privacy.
\newblock {\em Encyclopedia of Cryptography and Security}, pages 338--340,
  2011.

\bibitem{dwork2018}
C.~Dwork and V.~Feldman.
\newblock Privacy-preserving prediction.
\newblock In {\em Proceedings of the Conference on Learning Theory (COLT)},
  pages 1693--1702, 2018.

\bibitem{dwork2006}
C.~Dwork, F.~McSherry, K.~Nissim, and A.~Smith.
\newblock Calibrating noise to sensitivity in private data analysis.
\newblock In {\em Theory of cryptography conference}, pages 265--284. Springer,
  2006.

\bibitem{dwork2014algorithmic}
C.~Dwork, A.~Roth, et~al.
\newblock The algorithmic foundations of differential privacy.
\newblock {\em Foundations and Trends in Theoretical Computer Science},
  9(3--4):211--407, 2014.

\bibitem{dwork2016concentrated}
C.~Dwork and G.~N. Rothblum.
\newblock Concentrated differential privacy.
\newblock {\em arXiv preprint arXiv:1603.01887}, 2016.

\bibitem{dwork2010boosting}
C.~Dwork, G.~N. Rothblum, and S.~Vadhan.
\newblock Boosting and differential privacy.
\newblock In {\em 2010 IEEE 51st Annual Symposium on Foundations of Computer
  Science}, pages 51--60. IEEE, 2010.

\bibitem{fredrikson2015model}
M.~Fredrikson, S.~Jha, and T.~Ristenpart.
\newblock Model inversion attacks that exploit confidence information and basic
  countermeasures.
\newblock In {\em Proceedings of the 22nd ACM SIGSAC Conference on Computer and
  Communications Security}, pages 1322--1333, 2015.

\bibitem{he2016deep}
K.~He, X.~Zhang, S.~Ren, and J.~Sun.
\newblock Deep residual learning for image recognition.
\newblock In {\em CVPR}, 2016.

\bibitem{horn1994topics}
R.~A. Horn and C.~R. Johnson.
\newblock {\em Topics in matrix analysis}.
\newblock Cambridge University Press, 1994.

\bibitem{horn2012matrix}
R.~A. Horn and C.~R. Johnson.
\newblock {\em Matrix analysis}.
\newblock Cambridge University Press, 2012.

\bibitem{iyengar2019}
R.~Iyengar, J.~P. Near, D.~Song, O.~Thakkar, A.~Thakurta, and L.~Wang.
\newblock Towards practical differentially private convex optimization.
\newblock In {\em 2019 IEEE Symposium on Security and Privacy (SP)}, pages
  299--316. IEEE, 2019.

\bibitem{jayaraman2019}
B.~Jayaraman and D.~Evans.
\newblock Evaluating differentially private machine learning in practice.
\newblock In {\em Proceedings of the {USENIX} Security Symposium}, 2019.

\bibitem{kasiviswanathan2011can}
S.~P. Kasiviswanathan, H.~K. Lee, K.~Nissim, S.~Raskhodnikova, and A.~Smith.
\newblock What can we learn privately?
\newblock {\em SIAM Journal on Computing}, 40(3):793--826, 2011.

\bibitem{kifer2012}
D.~Kifer, A.~Smith, and A.~Thakurta.
\newblock Private convex empirical risk minimization and high-dimensional
  regression.
\newblock In {\em Proceedings of the Conference on Learning Theory (COLT)},
  2012.

\bibitem{krizhevsky2009}
A.~Krizhevsky.
\newblock Learning multiple layers of features from tiny images, 2009.

\bibitem{kruvzik2000bauer}
M.~Kru{\v{z}}{\'\i}k.
\newblock Bauer's maximum principle and hulls of sets.
\newblock {\em Calculus of Variations and Partial Differential Equations},
  11(3):321--332, 2000.

\bibitem{lecun1998}
Y.~LeCun and C.~Cortes.
\newblock The {MNIST} database of handwritten digits, 1998.

\bibitem{loosli2006}
G.~Loosli, S.~Canu, and L.~Bottou.
\newblock Training invariant support vector machines using selective sampling.
\newblock In L.~Bottou, O.~Chapelle, D.~{DeCoste}, and J.~Weston, editors, {\em
  Large Scale Kernel Machines}, pages 301--320. MIT Press, Cambridge, MA.,
  2007.

\bibitem{mcsherry2007mechanism}
F.~McSherry and K.~Talwar.
\newblock Mechanism design via differential privacy.
\newblock In {\em Proceedings of the Annual IEEE Symposium on Foundations of
  Computer Science (FOCS)}, pages 94--103. IEEE, 2007.

\bibitem{milli2019}
S.~Milli, L.~Schmidt, A.~D. Dragan, and M.~A.~W. Hardt.
\newblock Model reconstruction from model explanations.
\newblock In {\em Proceedings of the Conference on Fairness, Accountability,
  and Transparency}, 2019.

\bibitem{mironov2017}
I.~Mironov.
\newblock Renyi differential privacy.
\newblock In {\em Proceedings of the Computer Security Foundations Symposium
  (CSF)}, 2019.

\bibitem{mironov2019}
I.~Mironov, K.~Talwar, and L.~Zhang.
\newblock Renyi differential privacy of the sampled {G}aussian mechanism.
\newblock In {\em arXiv 1908.10530}, 2019.

\bibitem{nandi2019}
A.~Nandi and R.~Bassily.
\newblock Privately answering classification queries in the agnostic {PAC}
  model.
\newblock In {\em arXiv 1907.13553}, 2019.

\bibitem{nissim07}
K.~Nissim, S.~Raskhodnikova, and A.~Smith.
\newblock Smooth sensitivity and sampling in private data analysis.
\newblock In {\em STOC}, 2007.

\bibitem{papernot2016}
N.~Papernot, M.~Abadi, U.~Erlingsson, I.~Goodfellow, and K.~Talwar.
\newblock Semi-supervised knowledge transfer for deep learning from private
  training data.
\newblock {\em arXiv preprint arXiv:1610.05755}, 2016.

\bibitem{papernot2019}
N.~Papernot, S.~Chien, S.~Song, A.~Thakurta, and U.~Erlingsson.
\newblock Making the shoe fit: Architectures, initializations, and tuning for
  learning with privacy, 2019.

\bibitem{papernot2018}
N.~Papernot, S.~Song, I.~Mironov, A.~Raghunathan, K.~Talwar, and
  {\'U}.~Erlingsson.
\newblock Scalable private learning with {PATE}.
\newblock {\em arXiv preprint arXiv:1802.08908}, 2018.

\bibitem{sablayrolles2019}
A.~Sablayrolles, M.~Douze, Y.~Ollivier, C.~Schmid, and H.~J{\'e}gou.
\newblock White-box vs black-box: Bayes optimal strategies for membership
  inference.
\newblock In {\em Proceedings of the International Conference on Machine
  Learning (ICML)}, 2019.

\bibitem{shokri2017membership}
R.~Shokri, M.~Stronati, C.~Song, and V.~Shmatikov.
\newblock Membership inference attacks against machine learning models.
\newblock In {\em 2017 IEEE Symposium on Security and Privacy (SP)}, pages
  3--18. IEEE, 2017.

\bibitem{tramer2016}
F.~Tram{\`e}r, F.~Zhang, A.~Juels, M.~K. Reiter, and T.~Ristenpart.
\newblock Stealing machine learning models via prediction {API}s.
\newblock In {\em Proceedings of the {USENIX} Security Symposium}, pages
  601--618, 2016.

\bibitem{wang2017}
D.~Wang, M.~Ye, and J.~Xu.
\newblock Differentially private empirical risk minimization revisited: Faster
  and more general.
\newblock In {\em Advances in Neural Information Processing Systems}, 2017.

\bibitem{wolfe1969}
P.~Wolfe.
\newblock Convergence conditions for ascent methods.
\newblock {\em SIAM Review}, 11:226--000, 1969.

\bibitem{wu2017}
X.~Wu, F.~Li, A.~Kumar, K.~Chaudhuri, S.~Jha, and J.~F. Naughton.
\newblock Bolt-on differential privacy for scalable stochastic gradient
  descent-based analytics.
\newblock In {\em Proceedings of the International Conference on Management of
  Data}, 2017.

\bibitem{wu2018groupnorm}
Y.~Wu and K.~He.
\newblock Group normalization.
\newblock In {\em Proceedings of European Conference on Computer Vision}, 2018.

\bibitem{yeom2018}
S.~Yeom, I.~Giacomelli, M.~Fredrikson, and S.~Jha.
\newblock Privacy risk in machine learning: Analyzing the connection to
  overfitting.
\newblock In {\em CSF}, 2018.

\bibitem{zhu1997}
C.~Zhu, R.~H. Byrd, P.~Lu, and J.~Nocedal.
\newblock {L-BFGS-B}: Algorithm 778: {L-BFGS-B}, {FORTRAN} routines for large
  scale bound constrained optimization.
\newblock {\em ACM Transactions on Mathematical Software}, 23:550--560, 1997.

\end{thebibliography}
\end{document}